\documentclass{article}

    \usepackage[final]{neurips_2022}

\usepackage[utf8]{inputenc} % allow utf-8 input
\usepackage[T1]{fontenc}    % use 8-bit T1 fonts
\usepackage{hyperref}       % hyperlinks
\usepackage{url}            % simple URL typesetting
\usepackage{booktabs}       % professional-quality tables
\usepackage{amsfonts}       % blackboard math symbols
\usepackage{nicefrac}       % compact symbols for 1/2, etc.
\usepackage{microtype}      % microtypography
\usepackage{xcolor}         % colors

\usepackage{amsmath}
\usepackage{amssymb}
\usepackage{amsthm}
\usepackage{colortbl}
\usepackage{booktabs} % for professional tables
\usepackage{algorithm}
\usepackage{algorithmic}
\newtheorem{theorem}{Theorem}
\newtheorem{lemma}{Lemma}
\newtheorem{proposition}{Proposition}

\theoremstyle{definition}
\newtheorem{remark}{Remark}
\newtheorem{definition}{Definition}

\newcommand{\cD}{\mathcal{D}}

\newcommand{\cP}{\mathcal{P}}

\newcommand{\Nin}{N^{\mathrm{in}}}
\newcommand{\Nout}{N^{\mathrm{out}}}

\newcommand{\linner}{\left\langle}
\newcommand{\rinner}{\right\rangle}

\DeclareMathOperator*{\argmin}{arg\,min}

\DeclareMathOperator*{\E}{\mathbf{E}}
\newcommand{\re}{\mathbb{R}}
\renewcommand{\epsilon}{\varepsilon}

\newcommand{\revise}[1]{#1}

\title{Nearly Optimal Best-of-Both-Worlds Algorithms for Online Learning with Feedback Graphs}

\author{%
  Shinji Ito \\
  NEC Corporation, Tokyo, Japan \\
  RIKEN AIP, Tokyo, Japan \\
  \texttt{i-shinji@nec.com}\\
  \And
  Taira Tsuchiya \\
  Kyoto University, Kyoto, Japan \\
  RIKEN AIP, Tokyo, Japan \\
  \texttt{tsuchiya@sys.i.kyoto-u.ac.jp} \\
  \AND
  Junya Honda \\
  Kyoto University, Kyoto, Japan \\
  RIKEN AIP, Tokyo, Japan \\
  \texttt{honda@i.kyoto-u.ac.jp} \\
}

\begin{document}

\maketitle

\begin{abstract}
  This study considers online learning with general directed feedback graphs. For this problem, we present best-of-both-worlds algorithms that achieve nearly tight regret bounds for adversarial environments as well as poly-logarithmic regret bounds for stochastic environments. As \citet{alon2015online} have shown, tight regret bounds depend on the structure of the feedback graph: \textit{strongly observable} graphs yield minimax regret of $\tilde{\Theta}( \alpha^{1/2} T^{1/2} )$, while \textit{weakly observable} graphs induce minimax regret of $\tilde{\Theta}( \delta^{1/3} T^{2/3} )$, where $\alpha$ and $\delta$, respectively, represent the independence number of the graph and the domination number of a certain portion of the graph. Our proposed algorithm for strongly observable graphs has a regret bound of $\tilde{O}( \alpha^{1/2} T^{1/2} ) $ for adversarial environments, as well as of $ {O} ( \frac{\alpha (\ln T)^3 }{\Delta_{\min}} ) $ for stochastic environments, where $\Delta_{\min}$ expresses the minimum suboptimality gap. This result resolves an open question raised by \citet{erez2021towards}. We also provide an algorithm for weakly observable graphs that achieves a regret bound of $\tilde{O}( \delta^{1/3}T^{2/3} )$ for adversarial environments and poly-logarithmic regret for stochastic environments. The proposed algorithms are based on the \revise{follow-the-regularized-leader} approach combined with newly designed update rules for learning rates.
\end{abstract}

\section{Introduction}
\label{sec:intro}
In this paper,
we consider \textit{online learning with feedback graphs} \citep{mannor2011bandits},
a common generalization of the multi-armed bandit problem \citep{lai1985asymptotically,auer2002finite,auer2002nonstochastic} and the problem of prediction with expert advice \citep{littlestone1994weighted,freund1997decision}.
This problem is a sequential decision-making problem formulated with a directed \textit{feedback graph} $G = (V, E)$,
where $V = [K] := \{ 1, 2, \ldots, K \}$ is the set of arms or available actions,
and $E \subseteq V \times V$ represents the structure of feedback for choosing actions.
In each round of $t = 1, 2, \ldots, T$,
a player sequentially chooses an action $I_t \in V$ and then incurs the loss of $\ell_t(I_t)$,
where $\ell_t: V \to [0, 1]$ is a loss function chosen by the environment.
After choosing the action,
the player gets feedback of $\ell_t(j)$ for all $j$ such that the feedback graph $G$ has an edge from $I_t$ to $j$.
If $G$ consists of only self-loops,
i.e.,
if $E = \{ (i, i) \mid i \in V \}$,
the problem corresponds to a $K$-armed bandit problem.
If $G$ is a complete directed graph with self-loops, i.e., $E = V \times V$,
then the problem corresponds to a problem of prediction with expert advice.

\citet{alon2015online} have provided a characterization of minimax regrets
for the problem of online learning with feedback graphs.
They divide the class of all directed graphs into three categories.
For the first category,
called \textit{strongly observable graphs},
the minimax regret is $\tilde{\Theta}( \alpha^{1/2} T^{1/2} )$,
where $\alpha$ is the independence number of the graph $G$,
and $\tilde{\Theta}$ ignores poly-logarithmic factors in $T$ and $K$.
For the second category,
\textit{weakly observable graphs},
the minimax regret is $\tilde{\Theta} ( \delta^{1/3} T^{2/3} )$,
where $\delta$ represents the \textit{weakly dominating number}.
For the last category of \textit{unobservable graphs},
it is not possible to achieve sublinear regret,
which means that the minimax regret is $\Theta(T)$.
The definitions of categories of graphs and $\alpha$ and $\delta$ are given
in Section~\ref{sec:setting}.

Best-of-both-worlds (BOBW) algorithms \citep{bubeck2012best} have been studied
for the purpose of going beyond such minimax regret bounds;
they achieve sublinear regret for adversarial environments and,
as well,
have logarithmic regret bounds for stochastic environments.
The only BOBW algorithm for online learning with feedback graphs
has been proposed by \citet{erez2021towards}.
They have focused on the case in which $G$ is symmetric and all vertices have self-loops,
i.e.,
any edge $(i, j) \in E$ is accompanied by its reversed edge $(j, i) \in E$
and $(i, i) \in E$ for any $i \in V$.
Note that this is a special case of strongly observable graphs.
For this class of problems,
they provide an algorithm 
that achieves a regret bound of $\tilde{O}( \theta^{1/2} T^{1/2} )$ for adversarial environments,
and of ${O} \left( 
  \frac{\theta \mathrm{polylog}(T)}{\Delta_{\min}} 
  \right)$ for stochastic environments,
where $\theta~(\geq \alpha)$ is the \textit{clique covering number} of the graph $G$,
and $\Delta_{\min}$ is the minimum \textit{suboptimality gap} for the loss distributions. 
Their algorithm also works well for adversarially-corrupted stochastic environments,
achieving $O \left(
  \frac{\theta \mathrm{polylog}(T)}{\Delta_{\min}} 
  +
  \left(
  \frac{C \theta \mathrm{polylog}(T)}{\Delta_{\min}} 
  \right)^{1/2}
\right)$-regret,
where $C$ represents the total amount of corruption.

As \citet{erez2021towards} have pointed out, 
however,
their results leave room for improvement,
which is due to the fact that the clique covering number $\theta$ is significantly larger than
the independence number $\alpha$ in some cases.
Indeed,
there is an example such that $\alpha = 1$ while $\theta = K$,
as mentioned in Section~\ref{sec:setting}.
This means that regret bound depending on $\theta$ is not minimax optimal.
In response to this issue,
they have raised the question of whether it is possible to replace $\alpha$ with $\theta$ in their regret bounds.
Contributions of this study include a positive solution to this question.

\subsection{Contributions of this study}
\label{sec:contribution}
This study provides BOBW algorithms that achieve minimax regret (up to logarithmic factors)
for online learning with general feedback graphs.
Our contributions can be summarized as follows:
\begin{theorem}[strongly observable case, informal]
  \label{thm:strong-informal}
  For the problem with strongly observable graphs,
  an algorithm achieves $R_T = \tilde{O}( \alpha^{1/2} T^{1/2} )$
  for adversarial environments,
  $R_T = O \left(  \frac{\alpha (\ln T)^3 }{\Delta_{\min}}  \right)$ for stochastic environments,
  and 
  $R_T = O \left(  \frac{\alpha (\ln T)^3 }{\Delta_{\min}} + \left( 
    \frac{C \alpha (\ln T)^3 }{\Delta_{\min}} 
  \right)^{1/2} \right)$ for adversarially-corrupted stochastic environments,
  where $\alpha$ is the independence number of feedback graphs.
\end{theorem}
\begin{theorem}[weakly observable case, informal]
  \label{thm:weak-informal}
  For the problem with weakly observable graphs,
  an algorithm achieves $R_T = \tilde{O}( \delta^{1/3} T^{2/3} )$
  for adversarial environments,
  $R_T = O \left(  \frac{\delta (\ln T)^2 }{\Delta_{\min}^2} + \frac{ K' \ln T }{\Delta_{\min}} \right)$ for stochastic environments,
  and 
  $R_T = O \left(  \frac{\delta (\ln T)^2 }{\Delta_{\min}^2} + \left( 
    \frac{C^2 \delta (\ln T)^2 }{\Delta_{\min}^2} 
    \right)^{1/3}
    +
    \frac{K' \ln T}{\Delta_{\min}}
    +
    \left(
    \frac{C K' \ln T}{\Delta_{\min}}
    \right)^{1/2}
  \right)$ for adversarially-corrupted stochastic environments,
  where $\delta$ is the weakly dominating number of feedback graphs,
  and $K' (\leq K) $ is the number of vertices not covered by the weakly dominating set.
\end{theorem}
\revise{
\begin{remark}
The regret bound in Theorem~\ref{thm:weak-informal} for stochastic environments
include an $O\left(\frac{K' \ln T}{\Delta_{\min}}\right)$-term,
which is negligibly small compared to the other term $\frac{\delta(\ln T)^2}{\Delta^2_{\min}}$
when $T$ is sufficiently large.
However,
if $K'$ is larger than $\frac{\delta \ln T}{\Delta_{\min}}$,
this term can be dominant.
In such a case,
the regret upper bound may be improved by modifying the algorithm.
Roughly speaking,
by combining the approach to strongly observable case,
the $O\left(\frac{K' \ln T}{\Delta_{\min}}\right)$-term can be replaced with
an $O\left( \frac{\alpha' ( \ln T )^{3}}{\Delta_{\min}} \right)$-term,
where $\alpha'$ is the independence number of the subgraph consisting of vertices not dominated by the weakly dominating set.
If $\alpha' (\ln T)^2 \leq K'$,
the modified version provides a better regret bound.
Details of the modification are given Appendix~\ref{sec:weak-strong}.
\end{remark}
}

\begin{table}
  \caption{Regret upper bounds for online learning with feedback graphs.
  Note that regret bounds by \citet{erez2021towards} \revise{and \citet{rouyer2022near}} only apply to
  a special case of strongly observable graphs with self-loops.
  \revise{
  We also note that the graph consisting only of self-loops,
  which corresponds to the standard multi-armed bandit problem,
  is a special case of strongly observable graphs.
  }
  }
  \label{table:regret}
  \centering
  \begin{tabular}{llll}
    \toprule
    % observability
    feedback graph
    & reference & adversarial & stochastic \\
    \midrule
    strongly
    &
    \citep{alon2015online} & $\tilde{O}\left( \alpha^{1/2} T^{1/2} \right)$  & $ \tilde{O} \left( \alpha^{1/2} T^{1/2} \right)$     \\
    observable
    &
    \citep{erez2021towards}  & $\tilde{O}\left( \theta^{1/2} T^{1/2} \right) $& 
    $ O \left( \sum_{k} \frac{(\ln T)^4}{\Delta_{k}} \right)  $
    % ($|\Theta| = $)
    % $ \leq O \left( \frac{ \theta (\ln T)^4 }{\Delta_{\min}} \right)   $
    \\
    &
    \revise{\citep{rouyer2022near} }   & \revise{ $\tilde{O}\left( \alpha^{1/2} T^{1/2} \right)$ }    & 
    \revise{ $ O \left( \sum_{i \in S}\frac{(\ln T)^2}{\Delta_i} \right)   $      }
     \\
     &
    \textbf{[This work]} Theorem~\ref{thm:strong-informal}     &  $\tilde{O}\left( \alpha^{1/2} T^{1/2} \right)$    & 
    $ O \left( \frac{\alpha (\ln T)^{3}}{\Delta_{\min}} \right)   $      
     \\
    \midrule
    \revise{self-loops only}
    &
    \revise{\citep{zimmert2021tsallis}} & \revise{ ${O}\left( K^{1/2} T^{1/2} \right)$ }  & \revise{ $ {O} \left( \sum_{i: \Delta_i > 0} \frac{\ln T}{\Delta_i} \right)$}     \\
    \revise{
    (standard MAB)}
    &
    \revise{\textbf{[This work]} Theorem~\ref{thm:strong-informal} } & \revise{ $\tilde{O}\left( K^{1/2} T^{1/2} \right) $} & 
    \revise{ $ O \left( \frac{K(\ln T)^3}{\Delta_{\min}} \right) $ }
    \\
    \midrule
    \midrule
    weakly 
    & \citep{alon2015online} & $\tilde{O}\left( \delta^{1/3} T^{2/3} \right)$  & $\tilde{O}\left( \delta^{1/3} T^{2/3} \right)$  \\
    observable
    &
    \revise{\citep{kong2022simultaneously} }     & \revise{$\tilde{O}\left( K^{2/3} \delta^{1/3} T^{2/3} \right)$ }    & 
    \revise{ $ O \left( \delta^2 \frac{( \ln T )^{3/2}}{\Delta_{\min}^3} \right) $ }
    \\
    &
    \textbf{[This work]} Theorem~\ref{thm:weak-informal}     &  $\tilde{O}\left( \delta^{1/3} T^{2/3} \right)$    & 
    $ O \left( \frac{\delta (\ln T)^{2}}{\Delta_{\min}^2} + \frac{K' \ln T}{\Delta_{\min}} \right) $      
    \\
    \bottomrule
  \end{tabular}
\end{table}

Regret bounds for online learning with feedback graphs are summarized in Table~\ref{table:regret}.
Note that the regret bounds by \citet{erez2021towards} apply only to
the special case of strongly observable graphs that have self-loops for all vertices.
Their algorithm and regret bounds are stated with \textit{clique cover} $\{ V_k \}_{k=1}^{L}$ of $G$,
which is a partition of all vertices $V$ such that each $V_k$ is a clique.
The clique covering number $\theta$ of $G$ is the minimum size $L$ of clique covers.
Parameters $\Delta_k$ in Table~\ref{table:regret} are defined to be the minimum suboptimality gap among actions in $V_k$,
and the summation is taken over $k \in [L]$ such that $\Delta_k > 0$.
As the clique covering number $\theta$ is larger than or equal to the independence number $\alpha$ of $G$,
our adversarial regret bound of $\tilde{O}\left( \alpha^{1/2} T^{1/2} \right)$ for strongly observable cases is superior to that obtained by \citet{erez2021towards}
and is minimax optimal up to logarithmic factors.
Although our stochastic regret bound is also better than one by \citet{erez2021towards} in many cases,
it is not always so.
For example,
if $\alpha = \theta$ and $\Delta_{\min}$ is much smaller than many $\Delta_k$,
their regret may be better.
\revise{
  Note that the work by \citet{rouyer2022near},
  which proposes BOBW algorithms for strongly observable graphs with self-loops,
  has been published at NeurIPS 2022,
  independently of this study.
  While their algorithms achieve better regret bounds for a certain class of problem settings,
  our results have the advantage of being applicable to a wider range of problem settings,
  including directed feedback graphs without self-loops and adversarially corrupted stochastic settings.
  A more detailed discussion can be found in Appendix~\ref{sec:rouyer}
}

Our study includes the first \revise{nearly optimal} BOBW algorithm that can be applied to online learning with weakly observable graphs.
As shown in Table~\ref{table:regret},
the adversarial regret bounds obtained with the proposed algorithm match the minimax regret bound shown by \citet{alon2015online},
up to logarithmic factors.
Similarly to their algorithm,
the proposed algorithm uses a \textit{weakly dominating set} $D \subseteq V$ of $G$.
If $D$ is a weakly dominating set,
all elements 
in the set of vertices not dominated by $D$,
which is denoted by $V_2 \subseteq V$,
have self-loops.
Parameters $\delta$ and $K$ in the regret bounds are given by $\delta = |D|$ and $K' = |V_2|$.
The stochastic regret bound obtained with the proposed algorithm is also nearly tight.
In fact,
\citet{alon2015online} have shown
a regret lower bound of $\tilde{\Omega} \left(  \frac{\delta}{\Delta_{\min}^2}  \right) $
in the proof of Theorem 7 in their paper.
Further,
if vertices in $V_2$ are not connected by edges except for self-loops,
the problem is at least harder than the $K'$-armed bandit problem,
which leads to a regret lower bound of $\Omega\left( \frac{ K' \ln T }{\Delta_{\min}} \right)$.
\revise{
We note that,
just before the submission of this paper,
\citet{kong2022simultaneously} published a work on BOBW algorithms applicable to weakly observable graphs,
of which regret bounds are also included in Table~\ref{table:regret}.
}

\subsection{Techniques employed in this study}
The proposed algorithms are based on the follow-the-regularized-leader (FTRL) framework,
similarly to the algorithms by \citet{alon2015online} and \citet{erez2021towards}.
The main differences with existing methods are in the definitions of regularization functions
and update rules for learning rates.

For strongly observable cases,
we employ the Shannon entropy regularizer functions with a newly developed update rule for learning rates.
Most FTRL-based BOBW algorithms are realized by setting the learning rate adaptively to $t$ and/or observations.
On the other hand,
it is well known that FTRL with Shannon-entropy regularization corresponds to Exp3 algorithm \citep{auer2002nonstochastic} as discussed in, e.g., \citet[Example 28.3]{lattimore2018bandit}.
Since Exp3.G by \citet{alon2015online} achieves an independence-number-dependent regret bound for adversarial environments,
it is intuitively natural to expect that a variant of Exp3.G with adaptive learning rates can be used to achieve BOBW regret bounds.
However,
from the theoretical viewpoint,
it is necessary to express the regret depending on the arm-selection distribution $q_t$ to apply the \textit{self-bounding technique} \citep{gaillard2014second,zimmert2021tsallis},
which plays the central role in the BOBW analysis.

The proposed algorithm for weakly observable graphs
uses novel regularization functions consisting of Tsallis-entropy-based and Shannon-entropy-based regularization.
Intuitively,
we divide the vertices $V$ into the weakly dominated part $V_1$ and non-dominated part $V_2$,
and apply Shannon-entropy regularization to $V_1$
and Tsallis-entropy regularization to $V_2$.
We combine the FTRL method
with exploration using a uniform distribution over the weakly dominating set,
similarly to the approach by \citet{alon2015online}.
However,
we adjust exploration rates and learning rates in a carefully designed manner,
in contrast to the existing approach that employs fixed parameters.
The combination of the above techniques leads
to an entropy-dependent regret bound.
%  of $\tilde{O} \left( \delta^{1/3} \left(  \sum_{t=1}^T H(q_t) \right)^{2/3} \right)$.
By applying the self-bounding technique to this bound,
we obtain improved regret bounds for stochastic environments.

\section{Related work}
\label{sec:related}
Since \citet{bubeck2012best} initiated the study of best-of-both-worlds (BOBW) algorithms for
the multi-armed bandit (MAB) problem,
studies on BOBW algorithms have been extended to a variety of problem settings,
including the problem of prediction with expert advice \citep{gaillard2014second,luo2015achieving},
combinatorial semi-bandits \revise{\citep{zimmert2019beating,ito2021hybrid}},
linear bandits \citep{lee2021achieving},
episodic Markov decision processes \citep{jin2020simultaneously,jin2021best},
bandits with switching costs \citep{rouyer2021algorithm,amir2022better},
bandits with delayed feedback \citep{masoudian2022best},
\revise{online submodular optimization \citep{ito2022revisiting}},
and online learning with feedback graphs \citep{erez2021towards,kong2022simultaneously,rouyer2022near}.
Among these studies,
those using the follow-the-regularized-leader framework \citep{mcmahan2011follow} are particularly relevant to our work.
In an analysis of algorithms in this category,
we show regret bounds that depend on output distributions,
and we apply the self-bounding technique to derive BOBW regret bounds.
In applying this approach to partial feedback problems including MAB,
it has been shown that regularization based on the Tsallis entropy \citep{zimmert2021tsallis,zimmert2019beating} or 
the logarithmic barrier \revise{\citep{wei2018more,ito2021parameter,ito2022adversarially}} is useful.
By way of contrast,
our study employs regularization based on the Shannon entropy
and demonstrates for the first time that the self-bounding technique can be applied even with such regularization.

This study includes the regret bounds for stochastic environments with adversarial corruptions \citep{lykouris2018stochastic,gupta2019better,amir2020prediction},
which is an intermediate setting between stochastic and adversarial settings.
\citet{zimmert2021tsallis} have demonstrated that the self-bounding technique is also useful in deriving regret bounds for corrupted stochastic environments.
Typically,
when the self-bounding technique yields a regret bound of $O( \mathcal{R} )$ for stochastic environments,
it also yields a bound of $O( \mathcal{R} + \sqrt{ C \mathcal{R} } )$ for corrupted stochastic environments,
where $C$ represents the amount of corruption.
Examples of such results can be found in the literature,
e.g.,
that by \citet{zimmert2021tsallis}, \citet{erez2021towards}, and \citet{ito2021optimal}.
Our study follows the same strategy as these studies to obtain regret bounds for corrupted stochastic environments.

The problem of online learning with feedback graphs was formulated by \citet{mannor2011bandits},
and \citet{alon2015online} have provided a full characterization of minimax regret w.r.t.~this problem. 
Whereas these studies have considered adversarial models,
\citet{caron2012leveraging} have considered stochastic settings and proposed an algorithm with an $O(\ln T)$-regret bound.
In addition to these,
there can be found studies on such various extensions as models with
(uninformed) time-varying feedback graphs \citep{cohen2016online,alon2017nonstochastic},
stochastic feedback graphs \citep{kocak2016online,ghari2022online,esposito2022learning},
non-stationary environments \citep{lu2021stochastic-nonstationary},
and corrupted environments \citep{lu2021stochastic-corruptions}
as well as such improved algorithms as those with problem-dependent regret bounds \citep{hu2020problem}.

\section{Problem setting and known results}
\label{sec:setting}
Let $G = (V, E)$ be a directed graph with $V = [K] = \{ 1, 2, \ldots, K \}$ and $E \subseteq V \times V$,
which we refer to as a \textit{feedback graph}.
For each $i \in V$,
we denote the in-neighborhood and the out-neighborhood of $i$ in $G$ by
$\Nin (i)$ and $\Nout(i)$,
respectively,
i.e.,
$\Nin(i) = \{ j \in V \mid (j, i) \in E \}$ and
$\Nout(i) = \{ j \in V \mid (i, j) \in E \}$.

Before a game starts,
the player is given $G$.
For each round $t=1, 2, \ldots $,
the environment selects the loss functions $\ell_t : V \to [0, 1] $,
% of which values for $i \in V$ means the loss for choosing $i$ in the $t$-th round,
and the player then chooses $I_t \in V$ without knowing $\ell_t$,
where the value $\ell_t(i)$ represents the loss for choosing $i \in V$ in the $t$-th round.
After that,
the player incurs the loss of $\ell_t(I_t)$ and observes $\ell_t(j)$ for all $j \in \Nout (I_t)$.
Note that the player cannot observe the incurred loss if $I_t \notin \Nout(I_t)$.
The goal of the player is to minimize the sum of incurred loss.
To evaluate performance,
we use the regret $R_T$ defined by
\begin{align}
  R_T(i^*) = \E \left[
    \sum_{t=1}^T \ell_t(I_t)
    -
    \sum_{t=1}^T \ell_t(i^*)
  \right],
  \quad
  R_T
  =
  \max_{i^* \in V} R_T(i^*) ,
\end{align}
where the expectation is taken with respect to the randomness of $\ell_t$ and the algorithm's internal randomness.
The \textit{minimax regret} $R(G, T)$ is defined as the minimum over all randomized algorithms,
of the maximum of $R_T$ over all loss sequences $\{ \ell_t \}$.
\citet{alon2015online} have shown that the minimax regret can be characterized by
the notion of observability:
\begin{definition}[\citep{alon2015online}]
  \label{def:observability}
  A graph $G$ is \textit{observable} if $\Nin(i) \neq \emptyset$ holds for each $i \in V$.
  A graph $G$ is \textit{strongly observable} if $\{ i \} \subseteq \Nin (i) $
  or $V \setminus \{ i \} \subseteq \Nin (i)$ holds for each $i \in V$.
  A graph $G$ is \textit{weakly observable} if it is observable but not strongly observable.
\end{definition}
We further define the independence number $\alpha(G)$ and the weak domination number $\delta(G)$ as follows:
\begin{definition}
  \label{def:alphadelta}
  For a graph $G= (V, E)$,
  an \textit{independent set} $S \subseteq V$ is a set of vertices 
  such that $u, v \in S, u \neq v \Longrightarrow (u, v) \notin E$.
  The \textit{independence number} $\alpha(G)$ of $G$ is the size of its largest independent set.
  For a graph $G= (V, E)$,
  a \textit{weakly dominating set} $D \subseteq V$ is a set of vertices 
  such that $\{ i \in V \mid i \notin \Nout (i) \} \subseteq \bigcup_{i \in D} \Nout (i )  $.
  The \textit{weak domination number} $\delta(G)$ of $G$ is the size of its smallest weakly dominating set.
\end{definition}
\begin{remark}
  \revise{
  The definitions of weakly dominating set and weak domination number in this paper are slightly different from those by \citet{alon2015online}.
  % but the difference is negligible.
  % In fact,
  % for any weakly dominating set $S$ in the sense of our definition,
  % we can make it satisfy their definition of weak domination
  % by adding at most one vertex to $S$.
  % Furthermore,
  % any weakly dominating set in their definition is also a weakly dominating set in our definition.
  % These facts means that
  However,
  this difference is negligible as
  the gap between weak domination numbers in our definition and in theirs is at most one.
  Details are discussed in Appendix~\ref{sec:weak-domination-definition}.
  }
\end{remark}
The minimax regret can then be characterized as follows:
\begin{theorem}[\citep{alon2015online}]
  Let $G$ be a feedback graph with $|V| \geq 2$.
  Then,
  the minimax regret for $T \geq |V|^3$ is
  % \begin{description}
    % \item[(i)] 
    \textbf{(i)} $R(G, T) = \tilde{\Theta} ( \alpha^{1/2} T^{1/2} )$ if $G$ is strongly observable;
    % \item[(ii)] 
    \textbf{(ii)} $R(G, T) = \tilde{\Theta} ( \delta^{1/3} T^{2/3} )$ if $G$ is weakly observable;
    % \item[(iii)]
    \textbf{(iii)} $R(G, T) = {\Theta} ( T )$ if $G$ is not observable.
  % \end{description}
\end{theorem}
Following this statement by \citet{alon2015online},
we assume $T \geq |V|^3 = K^3$ in this paper.
%  unless otherwise stated.

Regret bounds by \citet{erez2021towards} listed in Table~\ref{table:regret} depend on 
the \textit{clique covering number} $\theta(G)$ of the feedback graph.
The clique covering number $\theta(G)$ is the minimum value of $N$ such that
there exists a \textit{clique cover} $\{ V_k \}_{k=1}^{N}$ for $G$ of size $N$.
A clique cover is a partition of vertices $V$ such that each $V_k$ is a clique,
i.e.,
$V_k \cap V_{k'} = \emptyset$ for all $k \neq k'$,
$\bigcup_{k=1}^\theta V_k = V$,
and 
$V_k \times V_k \subseteq E$ holds for any $k$.
While there exists an example such that $K = \theta(G)> \alpha(G) =1$,\footnote{
  For example,
  consider the graph $G = (V, E)$ given by $V = [K]$ and
  $E = \{ (i, j) \in V \times V \mid i \geq j \}$.
}
we always have $\theta(G)\ge \alpha(G)$, that is,
the clique covering number is at least the independence number.
In fact,
for any clique cover $\{ V_k \}_{k=1}^{N}$ and any independence set $S \subseteq V$,
two distinct elements in $S$ can never be in a single clique $V_k$,
which implies that $N \geq |S|$.

% For any feedback graph $G$,
% the clique covering number $\theta(G)$ is at least the independence number $\alpha(G)$.
% In fact,
% for any clique cover $\{ V_k \}_{k=1}^{N}$ and any independence set $S \subseteq V$,
% two distinct elements in $S$ can never be in a single clique $V_k$,
% which implies that $N \geq |S|$.
% Further,
% there exists an example such that $1 = \alpha(G) < \theta(G) = K$.
% While there exists an example such that $K=\theta(G)> \alpha(G) =1$,
% we always have $\theta(G)\ge \alpha(G)$, that is,
% the clique covering number is at least the independence number.
% In fact,
% for any clique cover $\{ V_k \}_{k=1}^{N}$ and any independence set $S \subseteq V$,
% two distinct elements in $S$ can never be in a single clique $V_k$,
% \begin{example}
%   \label{exa:clique}
%   Suppose that the feedback graph $G = (V, E)$ is given by $V = [K] = \{ 1, 2, \ldots, K \}$ and
%   $E = \{ (i, j) \in V \times V \mid i \geq j \}$.
%   Then we have $\alpha(G) = 1$ as every two vertices are connected.
%   On the other hand,
%   since $G$ has no clique of size larger than $1$,
%   the clique covering number $\theta(G)$ is equal to $|V| = K$.
% \end{example}

In this work,
we consider the \revise{\textit{adversarial regime with a self-bounding constraint}},
a comprehensive regime including stochastic settings,
adversarial settings,
and adversarially corrupted stochastic settings.
\begin{definition}[\revise{adversarial regime with a self-bounding constraint} \citep{zimmert2021tsallis}]
  \label{def:ARSB}
  Let $\Delta : V \rightarrow [0, 1]$ and $C \geq 0$.
  The environment is in an \textit{adversarial regime with a} $(\Delta, C, T)$ self-bounding constraint if
  it holds for any algorithm that
  \begin{align}
    \label{eq:defARSB}
    R_T \geq 
    \E \left[
      \sum_{t=1}^T \Delta(I_t) - C
    \right].
  \end{align}
\end{definition}
As has been shown by \citet{zimmert2021tsallis},
this regime includes (adversarially corrupted) stochastic settings.
Indeed,
if $\ell_t$ follows a distribution $\cD$ independently for $t=1,2,\ldots, T$
we have $R_T 
= \max_{i^* \in V} \E \left[ \sum_{t=1}^{T} ( \ell_t(I_t) - \ell_t(i^*) ) \right]
= \E \left[ \sum_{t=1}^{T} \Delta(I_t) \right]
$,
where we define $\Delta$ by $\Delta(i) = \E_{\ell \sim \cD} [ \ell(i) ] - \min_{i^* \in V} \E_{\ell \sim \cD} [ \ell(i^*) ] $.
This means that the environment is in an adversarial regime with a $(\Delta, 0, T)$ self-bounding constraint.
Further,
%if $\ell'_t$ follows $\cD$ for all $t$ and
%$\sum_{t=1}^T \max_{i \in [N]} | \ell_t(i) - \ell'_t(i) | \leq C$,
if $\ell_t$ satisfies $\sum_{t=1}^T \max_{i \in [N]} | \ell_t(i) - \ell'_t(i) | \leq C$ for
some $\ell'_t\sim \cD$,
the environment is in an adversarial regime with a $(\Delta, C, T)$ self-bounding constraint.
Note also that,
for any $\Delta: V \to [0, 1]$,
the adversarial regime with a $(\Delta, 2T, T)$ self-bounding constraint includes all the adversarial environments
since \eqref{eq:defARSB} clearly holds when $C = 2T$.

In this paper,
we assume that there exists $i^* \in V$ such that $\Delta(i^*) = 0$
and that $\Delta_{\min} := \min_{ i \in V \setminus \{ i^* \} } \Delta_{i} > 0$.
This implies that the optimal
arm
%solution
$i^*$ is assumed to be unique.
Similar assumptions were also made in previous works using the self-bounding technique \citep{gaillard2014second,luo2015achieving,wei2018more,zimmert2021tsallis,erez2021towards}.

\section{Preliminary}
The proposed algorithms are based on the follow-the-regularized-leader approach.
In this approach,
we define a probability distribution $p_t$ over $V$ as follows:
\begin{align}
  \label{eq:defFTRL}
  q_t \in \argmin_{ p \in \cP(V) } \left\{
    \sum_{s=1}^{t-1} \linner \hat{\ell}_s, p \rinner
    +
    \psi_t ( p )
  \right\},
  \quad
  p_t = 
  ( 1 - \gamma_t ) q_t
  +
  \gamma_t \mu_{U} ,
\end{align}
where $\cP(V) = \{ p : V \rightarrow [0, 1] \mid \sum_{i \in V} p(i) = 1 \}$
expresses the set of all probability distributions over $V$,
$\hat{\ell}_s$ is an unbiased estimator for $\ell_s$,
$\linner \ell, p \rinner = \sum_{ i \in V } \ell(i) p(i) $ represents the inner product,
$\psi_t: \cP \rightarrow \re$ is a convex regularizer function,
$\gamma_t \in [0, 0.5]$ is a parameter,
and
$\mu_U$ is the uniform distribution over a nonempty subset $U \subseteq V$,
i.e.,
$\mu_U(i) = 1 / |U| $ for $i \in U$ and
$\mu_U(i) = 0$ for $i \in V \setminus U$.
After computing $p_t$ defined by \eqref{eq:defFTRL},
we choose $I_t$ following $p_t$ 
so that
$\Pr [ I_t = i | p_t] = p_t(i)$.
We then observe $\ell_t(j)$ for each $j \in \Nout(I_t)$.
Based on these observations,
we set the unbiased estimator $\hat{\ell}_t: V \rightarrow \re$ by
\begin{align}
  \label{eq:defhatell}
  \hat{\ell}_t (i)
  =
  \frac{\ell_t(i)}{ P_t (i) } \mathbf{1}[ i \in \Nout (I_t) ]
  ,
  \quad
  P_t(i)
  =
  \sum_{j \in \Nin (i)} p_{t}(j) .
\end{align}
Let $D_t$ denote the \textit{Bregman divergence} with respect to $\psi_t$,
i.e.,
\begin{align}
  D_t (p, q) = \psi_t( p ) - \psi_t(q) - \linner  \nabla \psi_t(q), p - q \rinner .
\end{align}
We then have the following regret bounds:
\begin{lemma}
  \label{lem:FTRL}
  If $I_t$ is chosen by the above procedure,
  the regret is bounded by
  \begin{align}
    R_T 
    \leq
    \E \left[
      \sum_{t=1}^T
      \left(
      \gamma_t
      +
      \linner
      \hat{\ell}_t,
      q_t - q_{t+1} 
      \rinner
      -
      D_t(q_{t+1}, q_t )
      +
      \psi_{t}(q_{t+1})
      -
      \revise{
      \psi_{t+1}(q_{t+1})
      }
      \right)
    \right] 
    \nonumber
    \\
    +
    \psi_{T+1} ( \mu_{i^* } ) -
    \psi_1 (q_1),
    \label{eq:lemFTRL}
  \end{align}
  where $\mu_{i^*}(i) = 1$ if $i=i^*$
  and $\mu_{i^*}(i) = 0$ for $i \in V \setminus \{ i^* \}$.
\end{lemma}
This lemma can be shown by the standard analysis technique for FTRL,
e.g.,
given in Exercise 28.12 of the book by \citet{lattimore2018bandit},
combined with the fact that $\hat{\ell}_t$ defined by \eqref{eq:defhatell} is an unbiased estimator of $\ell_t$.
All omitted proofs will be given in the appendix.

We also introduce the following parameters $Q(i^*)$ and $Q$,
which will be used when applying self-bounding technique:
\begin{align}
  \label{eq:defQ}
  Q(i^*) = \sum_{t=1}^T (1 - q_t(i^*)),
  \quad
  \bar{Q}(i^*)
  =
  \E \left[
    % \sum_{t=1}^T (1 - q_t(i^*))
    Q(i^*)
  \right],
  \quad
  \bar{Q}  =  \min_{i^* \in V} \bar{Q}(i^*) .
\end{align}
We note that these values are clearly bounded as
$0 \leq \bar{Q} \leq \bar{Q}(i^*) \leq T$ for any $i^* \in V$.
In an adversarial regime with a self-bounding constraint,
the regret can be bounded from below,
as follows:
\begin{lemma}
  \label{lem:selfQ}
  In an adversarial regime with a self-bounding constraint given in Definition~\ref{def:ARSB},
  the regret is bounded as
  $R_T \geq \frac{\Delta_{\min}}{2} \bar{Q} - C$.
\end{lemma}
This lemma will be used to show poly-logarithmic regret in adversarial regime with a self-bounding constraint.

\section{Strongly observable case}
\label{sec:strong}
This section provides an algorithm achieving regret bounds in Theorem~\ref{thm:strong-informal}.
We set $U = V$ and
define $\psi_t$ using the Shannon entropy $H(p)$ as follows:
\begin{align}
  \label{eq:defShannon}
  \psi_t (p) = - {\beta_t} H(p),
  \quad
  \mbox{where}
  \quad
  H(p) = \sum_{i \in V} p(i) \ln \frac{1}{p(i)} ,
\end{align}
where $\beta_t > 0$ will be defined later.
If we choose
$\gamma_t = \min \left\{ \left( \frac{1}{\alpha T} \right)^{1/2}, \frac{1}{2} \right\}$
and $\beta_t = \frac{1}{2 \gamma_t}$ for all $t$,
the FTRL algorithm \eqref{eq:defFTRL} with \eqref{eq:defShannon} coincides
with the Exp3.G algorithm 
with the parameter setting given in Theorem 2 (i) by \citet{alon2015online}.
As shown by them,
this round-independent parameter setting leads to a regret bound
of $R_T = O( \alpha^{1/2} T^{1/2} \ln (KT) )$.

In this work,
we modify the update rule of $\beta_t$ and $\gamma_t$ as follows:
We set $\beta_1 = c_1 \geq 1$ and
update $\beta_t$ and $\gamma_t$ by
\begin{align}
  \label{eq:defbeta}
  \beta_{t+1}
  =
  \beta_{t}
  +
  \frac{ c_1 }{\sqrt{ 1 + (\ln K)^{-1} \sum_{s=1}^{t} a_s }},
  \quad
  \gamma_{t} = \frac{1}{ 2 \beta_t },
\end{align}
where $a_s$ is defined by $a_s = H(q_s)$.
In the following,
we will show the following regret bounds:
\begin{theorem}
  \label{thm:strong}
  If the feedback graph $G$ is strongly observable
  and has the independent number $\alpha = \alpha(G)$,
  the FTRL algorithm \eqref{eq:defFTRL} with $U=V$ and $\psi_t$ defined by \eqref{eq:defShannon} and \eqref{eq:defbeta}
  enjoys a regret bound of
  \begin{align}
    \label{eq:thmstrong1}
    R_T \leq 
    \hat{c} \cdot \max \left\{ \bar{Q}^{1/2}, 1\right\} ,
    ~
    % &
    \mbox{where}
    ~
    % Q = \min_{i^* \in V} \E \left[
    %   \sum_{t=1}^T ( 1 - q_t(i^*) )
    % \right]
    % \quad
    % \mbox{and}
    % \\
    % &
    \hat{c}
    =
    O \left(
      \left(
      \frac{\alpha \ln T \cdot \ln (c_1 K T ) }{c_1 \sqrt{\ln K}} 
      +
      c_1 
      \sqrt{\ln K}
      \right)
      \sqrt{\ln(KT)}
    \right) .
  \end{align}
  Consequently,
  we have 
  $R_T = O\left( \hat{c} \sqrt{T} \right) $
  in the adversarial regime
  and
  % \begin{align}
  %   \label{eq:thmstrong2}
  $
    % \left\{ 
      % \begin{array}{ll}
      %   O\left( \hat{c} \sqrt{T} \right) & (\mbox{adversarial regime}) \\
    R_T =
        O\left(  \frac{\hat{c}^2}{ \Delta_{\min} } + \sqrt{ \frac{C \hat{c}^2}{\Delta_{\min}} } \right)
    %      & (\mbox{stochastically constrained adversarial regime}) \\
    %   \end{array}
    % \right. .
  % \end{align}
  $
  in adversarial regimes with self-bounding constraints.
  % $R_T = \hat{c} \sqrt{ T }$ in the adversarial setting and
  % $R_T = \frac{\hat{c} T }{\Delta}$
\end{theorem}
When we set 
$c_1 = \Theta \left(  \sqrt{ \frac{ \alpha \ln T \cdot \ln (KT) }{ \ln K } }  \right)$,
$\hat{c}$ in this theorem is at most
$O \left( \sqrt{ \alpha \ln T \cdot (\ln (KT))^2 } \right)$,
which leads to the regret bounds in Theorem~\ref{thm:strong-informal}.
In the rest of this section,
we provide proof for Theorem~\ref{thm:strong}.

% If we choose $\beta_t = 1 / \eta$ and 
% For general round-dependent parameter setting,
% we have the following regret bound:
% from Lemma~\ref{lem:FTRL} and 
% the analysis for 
Let us start with the following lemma:
\begin{lemma}
  \label{lem:strong}
  If 
  $\psi_t$ is given by \eqref{eq:defShannon}
  with
  $\beta_t \geq 1$ and
  $\gamma_t \geq 1 / (2 \beta_t)$,
  the regret for the FTRL algorithm \eqref{eq:defFTRL} with $U = V$ is bounded as
  \begin{align}
    \label{eq:lemStrong}
  % $
    R_T
    \leq
    \E \left[
      \sum_{t=1}^T
      \left(
        \gamma_t
        +
        \frac{2}{\beta_t}
        \left(
          1 +
          4 \alpha \ln \frac{K^2}{4 \gamma_t}
        \right)
        +
        ( \beta_{t+1} - \beta_t) a_{t+1}
        % H( q_{t+1} ) 
      \right)
    \right]
    +
    \beta_1 \ln K ,
    % $
  \end{align}
  where $a_t = H(q_t)$ is the value of the Shannon entropy for $q_t$.
\end{lemma}
This lemma follows from Lemma~\ref{lem:FTRL} and 
the
technique used in the proof of \citet[Theorem 2]{alon2015online}.
% In the following,
% we denote $a_t = H( q_{t} )$ for the notational simplicity.
We note that $0 \leq a_t \leq \ln K$ and $a_1 = \ln K$.
From Lemma~\ref{lem:strong} and the update rules of parameters given by \eqref{eq:defbeta},
we obtain the following entropy-dependent regret bound:
% We then have the following:
\begin{proposition}
  \label{prop:strong}
  Suppose \eqref{eq:lemStrong} holds. 
  If $\beta_t$ and $\gamma_t$ are given by \eqref{eq:defbeta},
  % \begin{align}
  %   \label{eq:propStrong}
    $
    R_T
    \leq
    \tilde{c}
    \E \left[
      \sqrt{ \sum_{t=1}^T a_t }
    \right],
    $
    where
    $a_t = H(q_t)$ and
    $
    % \quad
    % \mbox{where}
    % \quad
    \tilde{c}
    =
    O \left(
      \frac{\alpha \ln T \cdot \ln (c_1 K T ) }{c_1 \sqrt{\ln K}} 
      +
      c_1 
      \sqrt{\ln K}
    \right) .
    $
  % \end{align}
\end{proposition}
\begin{proof}
  We will show the following two inequalities:
  \begin{align}
    \label{eq:propstrong1}
    \sum_{t=1}^T
    \left(
      \gamma_t
      +
      \frac{2}{\beta_t}
      \left(
        1 +
        4 \alpha \ln \frac{K^2}{4 \gamma_t}
      \right)
    \right)
    &
    =
    O \left(
      \frac{\alpha \ln T \cdot \ln (c_1 K^2 T ) }{c_1 \sqrt{\ln K}} 
      \sqrt{
        \sum_{t=1}^T a_t
      }
    \right),
    \\
    \sum_{t=1}^T
    (\beta_{t+1} - \beta_t) a_{t+1}
    &
    =
    O \left(
    {c_1 \sqrt{\ln K}}
    \sqrt{\sum_{t=1}^{T} a_t} 
    \right).
    \label{eq:propstrong2}
  \end{align}

  Let us first show \eqref{eq:propstrong1}.
  From the definition of $\gamma_t $ given in \eqref{eq:defbeta},
  we have
  $\mbox{[LHS of \eqref{eq:propstrong1}]} \leq 
    \sum_{t=1}^T
    \frac{1}{\beta_t}
    \left(
      3
      +
      8 \alpha \ln \frac{c_1 K^2 t}{2}
    \right)
    \leq
    \left(
      3
      +
      8 \alpha \ln \frac{c_1 K^2 T}{2}
    \right)
    \sum_{t=1}^T
    \frac{1}{\beta_t}
  $.
  From the definition of $\beta_t$ given by \eqref{eq:defbeta},
  $\beta_t$ is bounded as
  $
    \beta_t
    =
    c_1
    +
    \sum_{u=1}^{t-1}
    \frac{ c_1} {\sqrt{ 1 + (\ln K)^{-1} \sum_{s=1}^{u} a_s }}
    \geq
    \frac{ c_1 t } {\sqrt{ 1 + (\ln K)^{-1} \sum_{s=1}^{t} a_s } }
  $.
  We hence have
  $
    \sum_{t=1}^T
    \frac{1}{\beta_t}
    \leq
    \sum_{t=1}^T
    \frac{1}{ c_1 t } 
    \sqrt{ 1 + (\ln K)^{-1} \sum_{s=1}^{t} a_s } 
    \leq
    \frac{1 + \ln T}{ c_1 } 
    \sqrt{ 1 + (\ln K)^{-1} \sum_{t=1}^{T} a_t } 
    \leq
    O \left(
      \frac{\ln T}{c_1 \sqrt{\ln K}}
      \sqrt{ \sum_{t=1}^{T} a_t } 
    \right)
  $,
  where the last inequality follows from $a_1 = \ln K$.
  Combining the above inequalities,
  we obtain \eqref{eq:propstrong1}.

  Let us next show \eqref{eq:propstrong2}.
  From \eqref{eq:defbeta},
  we have
  $
    \mbox{[LHS of \eqref{eq:propstrong2}]}
    =
    \sum_{t=1}^T
    \frac{ c_1 }{
      \sqrt{1 + (\ln K)^{-1} \sum_{s=1}^{t} a_s }
    }
    \cdot
    a_{t+1}
    =
    {2 c_1 \sqrt{\ln K}}
    \sum_{t=1}^T
    \frac{ a_{t+1} }{
      \sqrt{\ln K + \sum_{s=1}^{t} a_s } + 
      \sqrt{\ln K + \sum_{s=1}^{t} a_s } 
    }
    \leq
    {2 c_1 \sqrt{\ln K}}
    \sum_{t=1}^T
    \frac{ a_{t+1} }{
      \sqrt{\sum_{s=1}^{t+1} a_s } + 
      \sqrt{\sum_{s=1}^{t} a_s } 
    }
    =
    {2 c_1 \sqrt{\ln K}}
    \sum_{t=1}^T
    \left(
      \sqrt{\sum_{s=1}^{t+1} a_s } -
      \sqrt{\sum_{s=1}^{t} a_s } 
    \right)
    =
    {2 c_1 \sqrt{\ln K}}
    \left(
      \sqrt{\sum_{s=1}^{T+1} a_s } -
      \sqrt{a_1} 
    \right)
    \leq
    {2 c_1 \sqrt{\ln K}}
    \sqrt{\sum_{t=1}^{T} a_t} ,
  $
  where inequalities follow from $a_{t} \leq a_1 = \ln K$.
  This proves \eqref{eq:propstrong2}.

  Inequalities \eqref{eq:propstrong1} and \eqref{eq:propstrong2} combined with \eqref{eq:lemStrong} lead to the regret bound in Proposition~\ref{prop:strong}.
\end{proof}
In addition,
$\sum_{t=1}^T a_t = \sum_{t=1}^T H(q_t)$ is bounded with $Q(i^*)$ defined in \eqref{eq:defQ},
as follows:
\begin{lemma}
  \label{lem:boundat}
  Suppose $a_t = H(q_t)$.
  For any $i^* \in V$,
  we have
  $
  \sum_{t=1}^T a_t
  \leq
  Q(i^*) \ln \frac{\mathrm{e}KT}{Q(i^*)}
  $.
\end{lemma}

We are now ready to prove Theorem~\ref{thm:strong}.\\
\textit{Proof of Theorem~\ref{thm:strong}}.
From Lemma~\ref{lem:boundat},
if $Q(i^*) \leq \mathrm{e}$,
we have $\sum_{t=1}^T a_t \leq \mathrm{e} \ln (KT)$ and
otherwise,
we have $\sum_{t=1}^T a_t \leq Q(i^*) \ln (KT)$.
Hence,
we have
$\sum_{t=1}^T a_t \leq \ln(KT) \cdot \max \left\{ \mathrm{e}, Q(i^*) \right\} $.
Combining this with Proposition~\ref{prop:strong},
we obtain \eqref{eq:thmstrong1}.
Since $\bar{Q} \leq T$,
we have $R_T \leq \hat{c} \sqrt{T}$ in adversarial regimes.

We next show
$ R_T = O\left(  \frac{\hat{c}^2}{ \Delta_{\min} } + \sqrt{ \frac{C \hat{c}^2}{\Delta_{\min}} } \right) $.
From Lemma~\ref{lem:selfQ},
\eqref{eq:defARSB} implies
if the environment satisfies a $(\Delta,C,T)$ self-bounding constraint \eqref{eq:defARSB},
% for $\Delta$ such that $\Delta_{\min} = \min_{i \in V \setminus \{ i^* \} }  \Delta(i) > 0$
we have
$R_T \geq \frac{\Delta_{\min}}{2} \bar{Q} - C$.
Combining this with Proposition~\ref{prop:strong} and Lemma~\ref{lem:boundat},
it holds for any $\lambda > 0$ that
\begin{align*}
  &
  R_T
  =
  (1 + \lambda) R_T
  - \lambda R_T
  \leq
  % \E \left[
  (1 + \lambda) \tilde{c}
  \sqrt{ \bar{Q} \ln (KT)}
  -
  \frac{\lambda \Delta_{\min}}{2} \bar{Q}
  % \right]
  +
  \lambda C 
  \\
  &
  \leq
  \frac{
  ((1 + \lambda) \tilde{c})^2 \ln (KT)
  }{2 \lambda \Delta_{\min}}
  +
  \lambda C 
  =
  \frac{
    \tilde{c}^2 \ln (KT)
  }{\Delta_{\min}}
  +
  \frac{1}{2\lambda}
  \frac{
    \tilde{c}^2 \ln (KT)
  }{\Delta_{\min}}
  +
  \frac{\lambda}{2}
  \left(
  \frac{
    \tilde{c}^2 \ln (KT)
  }{\Delta_{\min}}
  +
  2 C
  \right),
\end{align*}
where the first inequality follows from
Proposition~\ref{prop:strong}, Lemma~\ref{lem:boundat},
the condition of $Q(i^*) \geq \mathrm{e}$,
and \eqref{eq:selfQ}.
The second inequality follows from
$a \sqrt{x} - \frac{b}{2} x = \frac{a^2}{2 b} - \frac{1}{2}\left( \frac{a}{\sqrt{b}} - \sqrt{bx}  \right)
\leq \frac{a^2}{2b}$
which holds for any $a, b , x \geq 0$.
By choosing
$\lambda =  \sqrt{  
  \frac{
    \tilde{c}^2 \ln (KT)
  }{\Delta_{\min}}
  /
  \left(
  \frac{
    \tilde{c}^2 \ln (KT)
  }{\Delta_{\min}}
  +
  2 C
  \right)
}$,
we obtain
$ R_T = O\left(  \frac{\hat{c}^2}{ \Delta_{\min} } + \sqrt{ \frac{C \hat{c}^2}{\Delta_{\min}} } \right) $.
\qed

\section{Weakly observable case}
\label{sec:weak}
This section provides an algorithm achieving regret bounds in Theorem~\ref{thm:weak-informal}.
Let $D$ be a weakly dominating set,
defined in Definition~\ref{def:alphadelta},
and let $V_1 = \bigcup_{i \in D} \Nout(i)
% \{ i \in V \mid i \in \Nin(i) \}
$,
$V_2 = V \setminus V_1$.
We consider here the FTRL approach given by \eqref{eq:defFTRL} with $U = D$ and
regularizer functions
defined as
\begin{align}
  \nonumber
  \psi_t (p)
  &
  =
  \beta_t
  \sum_{i \in V_1} h (p(i))
  +
  \sum_{i \in V_2} \sqrt{t} g (p(i)),
  \\
  &
  \mbox{where}
  \quad
  h(x) = x \ln x + (1 - x) \ln (1 - x) ,
  \quad
  g( x ) = - 2 \sqrt{x} - 2 \sqrt{1 - x}.
  \label{eq:defpsiweak}
\end{align}
The regularization with $h(x)$ for $V_1$ is a variant of Shannon-entropy regularization,
which can be considered as a modification of the approach of the Exp3.G by \citet{alon2015online}.
The remaining part defined with $g(x)$ for $V_2$ is a modification of the approach used in
the Tsallis-INF algorithm by \citet{zimmert2021tsallis},
which is a BOBW algorithm for MAB problems.
Intuitively,
approaches for MAB work well for vertices in $V_2$ as they have self-loops,
i.e.,
choosing actions in $V_2$ admits bandit feedback.

Let us define parameters $\gamma_t$ and $\beta_t$ by
$\beta_1 = \max\{ c_2 , 8 |D| \}$ and
\begin{align}
  \label{eq:defgammabetaweak}
	\gamma'_{t}
	=
  \frac{1}{4}
	\frac{c_1 b_t}{ c_1 +  \left( \sum_{s=1}^t b_s \right)^{1/3} },
	\quad
	\beta_{t+1}
	=
	\beta_t
	+
	\frac{c_2 b_t}{ \gamma_t' \left(
    c_1
    +
		\sum_{s=1}^{t-1} \frac{b_s a_{s+1}}{ \gamma_s' }
	\right)^{1/2} },
	\quad
	\gamma_{t}
	=
	\gamma'_{t}
	+
	\frac{2 |D|}{\beta_t},
\end{align}
where $c_1, c_2 > 0$ are input parameters such that
$c_1 \geq 2 \ln K$ and
with $\{ a_t \}$ and $\{ b_t \}$ are defined by
\begin{align}
  \label{eq:defatbtweak}
  a_t = 
  -
  \sum_{i \in V_1} h( q_t(i) ),
  \quad
  b_t = 
  \sum_{i \in V_1} q_t(i) ( 1 - q_t(i) ) .
\end{align}
Note that $a_t$ and $\hat{c}$ used in this Section~\ref{sec:weak}
are different from those defined in Section~\ref{sec:strong}.
We then have the following regret bounds:
\begin{theorem}
  \label{thm:weak}
  If the feedback graph $G$ is weakly observable,
  the FTRL algorithm \eqref{eq:defFTRL} with $U=D$ and $\psi_t$ defined by \eqref{eq:defpsiweak} and \eqref{eq:defgammabetaweak}
  enjoys a regret bound of
  \begin{align}
    \nonumber
    R_T 
    &
    \leq 
    \hat{c} \cdot \max \left\{ \bar{Q}^{2/3}, c_1^2 \right\} 
    +
    O \left(
    ( |V_2| \ln T \cdot \bar{Q}  )^{1/2}
    \right)
    \quad
    \mbox{where}
    \\
    &
    \quad
    \hat{c}
    =
    O \left(
      c_1
      +
      \frac{1}{\sqrt{c_1}}
      \left(
        \frac{ |D| \ln T}{c_2}
        +
        c_2
      \right)
      \sqrt{\ln (KT)}
    \right) .
    \label{eq:thmweak1}
  \end{align}
  Consequently,
  if $T \geq K^3$,
  we have $R_T = O \left( \hat{c} T^{2/3} \right)$ in the adversarial regime and
  \begin{align}
    \label{eq:thmweak2}
    R_T = 
        O\left( 
          \frac{\hat{c}^3}{ \Delta_{\min}^2 } + \left(\frac{C^2 \hat{c}^3}{\Delta_{\min}^2} \right)^{1/3}
          +
          \frac{|V_2|\ln T }{ \Delta_{\min} } + \sqrt{\frac{C |V_2| \ln T }{\Delta_{\min}}} 
        \right) 
  \end{align}
  in adversarial regimes with self-bounding constraints.
\end{theorem}
We obtain
$
\hat{c}
=
O \left(
  \left(
    |D| \ln T \cdot \ln (KT)
  \right)^{1/3}
\right) 
$ 
by setting
$c_1 = \Theta\left( \left( |D| \ln T \cdot \ln (KT)  \right)^{1/3} \right) $
and
$c_2 = \Theta \left( \sqrt{ |D| \ln T } \right)$.
By using a weakly dominating set $D$ such that $|D| = O(\delta(G))$,
we obtain the regret bounds in Theorem~\ref{thm:weak-informal}.
The remainder of this section is dedicated to the proof of Theorem~\ref{thm:weak}.

We start with
the following regret bound:
\begin{lemma}
  \label{lem:decomposeweak}
  If $\psi_t$ is given by \eqref{eq:defpsiweak}
  and if $\gamma_t \geq \frac{2 |D|}{\beta_{t}}$,
  we have
  $
    R_T
    \leq
    R_T^{(1)}
    +
    R_T^{(2)}
    +
    a_1 \beta_1
  $,
  where
    % \quad
    % \mbox{where}
    % \\
  \begin{align}
    \label{eq:defR1}
    &
    R_T^{(1)}
    =
    O
    \left(
      \E
      \left[
        \sum_{t=1}^T
        \left(
          \gamma_t
          +
          % \sum_{i \in V_2} 
          % \left(
          \frac{|D| b_t }{ \gamma_t \beta_t }
          +
          ( \beta_{t+1} - \beta_t ) a_{t+1}
        \right)
        % +
        % a_1 \beta_1
      \right]
    \right)  ,
    \\
    &
    % \quad
    R_T^{(2)}
    =
    O\left(
      \E
      \left[
        \sum_{t=1}^T
        % \left(
          \frac{1}{\sqrt{t}}
          \sum_{i \in V_2} \sqrt{q_t(i) (1 - q_t(i))}
          % \right)
        % \right)
      \right]
    \right) ,
    \label{eq:defR2}
  \end{align}
  with $\{ a_t \}$ and $\{ b_t \}$ defined by \eqref{eq:defatbtweak}.
\end{lemma}
When showing \eqref{eq:defR1} and \eqref{eq:defR2},
we use techniques used in the proofs of \citet[Theorem 2]{alon2015online} and
of \citet[Lemma 11]{zimmert2021tsallis}.
We then have the following bound:
\begin{proposition}
  \label{prop:weak}
  If $\gamma_t$ and $\beta_t$ are given by \eqref{eq:defgammabetaweak},
  $R^{(1)}_T$ satisfying \eqref{eq:defR1} is bounded as
  % we have
  \begin{align}
    R_T^{(1)}
    =
    O
    \left(
      \E \left[
      c_1 
      B_T^{2/3}
      +
      \tilde{c}
      \sqrt{
        c_1^2
        +
        \left( \ln K  + A_T \right)
        \left( c_1 + B_T^{1/3} \right)
      }
      \right]
    \right)
    ,
  \end{align}
  where 
  $A_T = \sum_{t=1}^T a_t$,
  $B_T = \sum_{t=1}^T b_t$
  and 
  $\tilde{c} = O \left( \frac{1}{\sqrt{c_1}} \left( \frac{|D| \ln T}{c_2} + {c_2} \right) \right)$.
\end{proposition}
% In the following,
% we show  ....
Values of $A_T$ and $B_T$ in this proposition can be bounded with $Q(i^*)$ defined in \eqref{eq:defQ},
as follows:
\begin{lemma}
  \label{lem:boundATBT}
  $A_T$ and $B_T$ defined in Proposition~\ref{prop:weak} satisfy
  $A_T \leq 2 Q(i^*) \ln \frac{\mathrm{e} KT}{ Q(i^*) }$
  and $B_T \leq 2 Q(i^*)$ .
\end{lemma}
Further, $R^{(2)}_T$ in Lemma~\ref{lem:decomposeweak} can be bounded with $\bar{Q}$ as follows:
\begin{lemma}
  \label{lem:boundsqrt}
  $R^{(2)}_T$ satisfying \eqref{eq:defR2} is bounded as
  $R^{(2)}_T = O\left( \sqrt{ |V_2| \ln T \cdot \bar{Q}  } \right)$.
\end{lemma}
\textit{Proof of Theorem~\ref{thm:weak}}.
From Proposition~\ref{prop:weak} and Lemma~\ref{lem:boundATBT},
if $\bar{Q} \geq c_1^3$,
we have
\begin{align*}
  R_T^{(1)}
  &
  =
  O \left(
    \E \left[
    c_1 {Q}(i^*)^{2/3}
    +
    \tilde{c} 
    \sqrt{ {Q}(i^*) \ln (KT) {Q}(i^*)^{1/3} }
    \right]
  \right)
  \leq
  O \left(
    \left(
    c_1 
    % Q(i^*)^{2/3}
    +
    \tilde{c} 
    \sqrt{\ln (KT)}
    \right)
    \bar{Q}^{2/3}
    % \sqrt{ Q(i^*) \ln (KT) Q(i^*)^1/3 }
  \right) ,
\end{align*}
where the inequality follows from Jensen's inequality.
Hence,
there exists $\hat{c}$ such that
$R_T^{(1)} \leq \hat{c} \cdot \bar{Q}^{2/3} $
and $\hat{c} = O \left( c_1 + \tilde{c}\sqrt{ \ln (KT)} \right) $.
Combining this with Lemma~\ref{lem:boundsqrt},
we obtain \eqref{eq:thmweak1}.
As we have $\bar{Q}\leq T$,
in adversarial regimes with $T \geq K^3$,
it follows from \eqref{eq:thmweak1} that
$R_T = O \left( \hat{c} \cdot \max\{ T^{2/3}, c_1^2 \} + (K \ln T \cdot T )^{1/2} \right)
= O \left( \hat{c} \cdot T^{2/3} \right) $,
where the second equality follows from the $T \geq K^3$.
Let us next show \eqref{eq:thmweak2}.
From \eqref{eq:thmweak1} and Lemma~\ref{lem:selfQ},
for any $\lambda \in (0, 1]$,
we have
\begin{align*}
  R_T
  =
  (1 + \lambda)R_T - \lambda R_T
  =
  O\left(
    ( 1 + \lambda ) \hat{c} \cdot \bar{Q}^{2/3}
    +
    (1 + \lambda) (|V_2| \ln T \cdot \bar{Q})^{1/2}
    -
    \lambda \Delta_{\min} \bar{Q}
    +
    \lambda C
  \right) .
\end{align*}
By an argument similar to the proof of Theorem~\ref{thm:strong},
we have
% \begin{align*}
  $
  (1 + \lambda) (|V_2| \ln T \cdot \bar{Q})^{1/2}
  -
  \lambda \Delta_{\min} \bar{Q}
  =
  O \left(
    \left(1 + \frac{1}{\lambda} \right)
    \frac{|V_2| \ln T}{\Delta_{\min}} 
  \right) .
  $
% \end{align*}
We also have
% \begin{align}
  $
  ( 1 + \lambda ) \hat{c} \cdot \bar{Q}^{2/3}
  -
  \lambda \Delta_{\min} \bar{Q}
  % &
  =
  % O \left(
    \left( \frac{(1 + \lambda)^3 \hat{c}^3}{ \lambda^2 \Delta_{\min}^2 } \right)^{1/3}
    \left( \lambda \Delta_{\min} \bar{Q} \right)^{2/3}
    -
    \lambda \Delta_{\min} \bar{Q}
  % \right)
  % \\
  % &
  =
  O \left(
    \frac{(1 + \lambda)^3 \hat{c}^3}{ \lambda^2 \Delta_{\min}^2 }
  \right)
  =
  O \left(
    \left(
      1 + \frac{1}{\lambda^2}
    \right)
    \frac{\hat{c}^3}{ \Delta_{\min}^2 }
  \right),
  $
% \end{align}
where the second equality follows from
$x^{1/3} y^{2/3} \leq \frac{1}{3} x + \frac{2}{3} y$ that holds for any $x, y \geq 0$.
Combining these inequalities,
we obtain
% \begin{align}
  $
  R_T =
  O \left(
    \left(
      1 + \frac{1}{\lambda^2}
    \right)
    \frac{\hat{c}^3}{ \Delta_{\min}^2 }
    +
    \left(1 + \frac{1}{\lambda} \right)
    \frac{|V_2| \ln T}{\Delta_{\min}} 
    +
    \lambda C
  \right).
  $
% \end{align}
By choosing $\lambda$ that minimizes the RHS,
we obtain \eqref{eq:thmweak2}.
\qed

\section*{Acknowledgment}
TT was supported by JST, ACT-X Grant Number JPMJAX210E, Japan and JSPS, KAKENHI Grant Number JP21J21272, Japan.
JH was supported by JSPS, KAKENHI Grant Number JP21K11747, Japan.

\bibliographystyle{abbrvnat}
\bibliography{reference}

\newpage
\appendix

\section{Omitted proofs}
\subsection{Proof of Lemma~\ref{lem:FTRL}}
\begin{proof}
  From the definition of the algorithm,
  we have
  \begin{align}
    \nonumber
    R_T(i^*)
    &
    =
    \E \left[
    \sum_{t=1}^T 
    \ell_t( I_t )
    -
    \sum_{t=1}^T 
    \ell_t( i^* )
    \right]
    =
    \E \left[
    \sum_{t=1}^T 
    \linner
    \ell_t ,
    p_t - \mu_{i^*}
    \rinner
    \right]
    \\
    &
    \nonumber
    =
    \E \left[
    \sum_{t=1}^T 
    \linner
    \ell_t ,
    q_t - \mu_{i^*}
    \rinner
    +
    \sum_{t=1}^T 
    \gamma_t
    \linner
    \ell_t ,
    \mu_{U} - q_t
    \rinner
    \right]
    \leq
    \E \left[
    \sum_{t=1}^T 
    \linner
    {\ell}_t ,
    q_t - \mu_{i^*}
    \rinner
    +
    \sum_{t=1}^T 
    \gamma_t
    \right]
    \\
    &
    =
    \E \left[
    \sum_{t=1}^T 
    \linner
    \hat{\ell}_t ,
    q_t - \mu_{i^*}
    \rinner
    +
    \sum_{t=1}^T 
    \gamma_t
    \right],
    \label{eq:lemFTRL1}
  \end{align}
  where the second equality follows from $I_t \sim p_t$,
  the third equality follows from the second part of \eqref{eq:defFTRL},
  the first inequality follows from
  $\linner \ell_t, \mu_U - q_t \rinner \leq \linner \ell_t, \mu_U \rinner \leq 1$,
  and the last equality follows from the fact that $\hat{\ell}_t$ is an unbiased estimator for $\ell_t$.
  Further,
  from Exercise 28.12 of the book by \citet{lattimore2018bandit},
  we have
  \begin{align*}
    &
    \sum_{t=1}^T 
    \linner
    \hat{\ell}_t ,
    q_t - \mu_{ i^* }
    \rinner
    \\
    &
    \leq
    \sum_{t=1}^T \left(
      \linner
      \hat{\ell}_t,
      q_t - q_{t+1} 
      \rinner
      -
      D_t(q_{t+1}, q_t )
      +
      \psi_{t}(q_{t+1})
      -
      \psi_{t+1}(q_{t+1})
    \right)
    +
    \psi_{T+1} ( \mu_{i^* } ) -
    \psi_1 (q_1).
  \end{align*}
  Combining this with \eqref{eq:lemFTRL1},
  we obtain \eqref{eq:lemFTRL}.
\end{proof}

\subsection{Proof of Lemma~\ref{lem:selfQ}}
\begin{proof}
Suppose that \eqref{eq:defARSB} holds
with $\Delta: V\rightarrow \re$
such that $\Delta(i) \geq \Delta_{\min}$ for all $i \in [K] \setminus \{ i^* \}$.
% Then for any $i^* \in [K]$,
The regret is then bounded as
% For any $i^*$
% As we have $\gamma_t \leq \frac{1}{2}$ from the definition of $\gamma_t$
% in \eqref{eq:defbeta},
% it holds that
\begin{align}
  \nonumber
  R_T
  &
  \geq
  \E \left[
    \sum_{t=1}^T \Delta(I_t)
  \right]
  -
  C
  =
  \E \left[
    \sum_{t=1}^T \sum_{i \in V} \Delta(i) p_t(i) 
  \right]
  -
  C
  \\
  &
  \geq
  \E \left[
    \sum_{t=1}^T \sum_{i \in V} \Delta(i)
    (1 - \gamma_t)q_t(i) 
  \right]
  -
  C
  \geq
  \E\left[ 
    \frac{\Delta_{\min}}{2} Q(i^*) 
  \right]
  -C
  \geq
  \frac{\Delta_{\min}}{2}
  \bar{Q}
  -C ,
  \label{eq:selfQ}
\end{align}
where the first inequality follows from \eqref{eq:defARSB},
the first equality follows from $I_t \sim p_t$,
the second inequality follows from the definition of $p_t$ given in \eqref{eq:defFTRL},
and the third and last inequalities follow from the assumption of $\gamma_t \leq \frac{1}{2}$ and
the definitions of $Q(i^*)$ and $\bar{Q}$ given in \eqref{eq:defQ}.
This completes the proof of Lemma~\ref{lem:selfQ}.
\end{proof}

\subsection{Proof of Lemma~\ref{lem:strong}}
We use the following lemma to analyze the right-hand sided of \eqref{eq:lemFTRL}.
\begin{lemma}
  \label{lem:boundBregStrong}
  If $\psi_t$ is given by \eqref{eq:defShannon},
  it holds for any $\ell: V \rightarrow \re$ and
  $p, q \in \cP(V)$ that
  \begin{align}
    \label{eq:boundBregStrong}
    \linner \ell, p - q \rinner 
    -
    D_t ( q , p )
    \leq
    \beta_t
    \sum_{i \in V} p(i)
    \xi \left(
      \frac{\ell(i)}{\beta_t}
    \right),
    \quad
    \mbox{
    where
    }
    \quad
    \xi(x) = \exp(-x) + x - 1 .
    % \left( \exp\left( - \frac{\ell(i)}{\beta_t} \right) + \frac{ \ell(i)}{\beta_t} - 1 \right).
  \end{align}
\end{lemma}
\begin{proof}
  The derivative of the LHS of \eqref{eq:boundBregStrong} w.r.t. $q(i)$ is expressed as
  \begin{align}
    \frac{\partial}{\partial q(i)}
    \left(
    \linner \ell, p - q \rinner 
    -
    D_t ( q , p )
    \right)
    =
    -
    \ell(i)
    -
    \beta_t \left(
      \ln q(i)
      -
      \ln p(i)
    \right).
    \label{eq:deriveStrong}
  \end{align}
  As the LHS of \eqref{eq:boundBregStrong} is concave in $q$,
  its maximum subject to $q: V \rightarrow \re_{>0}$ is attained when
  % it is achieved when 
  the values of \eqref{eq:deriveStrong} are equal to zero,
  i.e.,
  $q(i) = q^*(i) := p(i) \exp\left( - \frac{\ell(i)}{\beta_t} \right)$.
  Hence,
  we have
  \begin{align*}
    &
    \linner \ell, p - q \rinner 
    -
    D_t ( q , p )
    \leq
    \linner \ell, p - q^* \rinner 
    -
    D_t ( q^* , p )
    \\
    &
    =
    \sum_{i \in V} 
    \left(
    \ell(i) ( p(i) - q^*(i) )
    -
    \beta_t
    \left(
    q^*(i) \ln q^*(i)
    -
    p(i) \ln p(i)
    -
    (\ln p(i) + 1)
    (q^*(i) - p(i))
    \right)
    \right)
    \\
    &
    =
    \sum_{i \in V} 
    \left(
    \ell(i) p(i) 
    -
    \beta_t
    \left(
      q^*(i)
      \ln p(i)
    -
    p(i) \ln p(i)
    -
    (\ln p(i) + 1)
    (q^*(i) - p(i))
    \right)
    \right)
    \\
    &
    =
    \sum_{i \in V} 
    \left(
    \ell(i) p(i) 
    +
    \beta_t
    \left(
    (q^*(i) - p(i))
    \right)
    \right)
    =
    \beta_t
    \sum_{i \in V} p(i) \left( \exp\left( - \frac{\ell(i)}{\beta_t} \right) + \frac{ \ell(i)}{\beta_t} - 1 \right)
    \\
    &
    =
    \beta_t
    \sum_{i \in V} p(i)
    \xi \left(
      \frac{\ell(i)}{\beta_t}
    \right),
  \end{align*}
  where the first equality follows from the definition of the Bregman divergence and \eqref{eq:defShannon},
  the second equality follows from $\ln q^*(i) = \ln p(i) - \frac{\ell(i)}{\beta_t}$,
  and the fourth inequality follows from $q^*(i) = p(i) \exp \left( - \frac{\ell(i)}{\beta_t} \right)$.
  This complete the proof of Lemma~\ref{lem:boundBregStrong}.
\end{proof}
Note that
as we have $\exp(-x) \leq 1 - x + x^2$ for any $x \geq -1$,
the function $\xi$ defined in \eqref{eq:boundBregStrong} satisfies $\xi(x) \leq x^2$ for any $x \geq - 1$.
Hence,
Lemma~\ref{lem:boundBregStrong} implies that
$
  \linner \ell, p - q \rinner 
  -
  D_t ( q , p )
  \leq
  \beta_t
  \sum_{i \in V} p(i)
  \xi \left(
    \frac{\ell(i)}{\beta_t}
  \right)
  \leq
  \frac{1}{\beta_t}
  \sum_{i \in V} p(i)
  \ell(i)^2
$
holds
for any $\ell: V \rightarrow [- \beta_t, \infty )$.

Denote $S = \{ i\in V \mid i \notin \Nin(i) \}$.
From Lemma~\ref{lem:boundBregStrong} and
the argument by \citet[Lemma 4, Theorem 2]{alon2015online},
we have
\begin{align}
  \nonumber
  &
  \E
  \left[
  \linner \hat{\ell}_t, q_{t} - q_{t+1} \rinner 
  -
  D_t ( q_{t+1} , q_t )
  \right]
  =
  \E
  \left[
  \linner \hat{\ell}_t - \bar{\ell}_t \cdot \mathbf{1}, q_{t} - q_{t+1} \rinner 
  -
  D_t ( q_{t+1} , q_t )
  \right]
  \\
  &
  \leq
  \beta_t
  \sum_{i \in V} q_t(i)
  \xi \left(
    \frac{\hat{\ell}_t(i) - \bar{\ell}_t}{\beta_t}
  \right)
  \leq
  \frac{1}{\beta_t}
  \left(
  \sum_{i \in S} q_t(i) (1 - q_t(i)) \hat{\ell}_t(i)^2
  +
  \sum_{i \in V \setminus S } q_t(i) \hat{\ell}_t(i)^2
  \right),
  \label{eq:boundStabilityStrong0}
\end{align}
where $\bar{\ell}_t$ is defined in a way similar to by \citet[Lemma 4]{alon2015online},
the first inequality follows from Lemma~\ref{lem:boundBregStrong}
and the last inequality follows from the definition of $\bar{\ell}_t$ and
the inequality $\xi(x) \leq x^2$ that holds for $x \geq - 1$.
The first term of the right-hand side of \eqref{eq:boundStabilityStrong0}
can be bounded as
\begin{align}
  \nonumber
  &
  \E \left[
    \sum_{i \in S} q_t(i) (1 - q_t(i)) \hat{\ell}_t(i)^2
  \right]
  =
  \E \left[
    \sum_{i \in S} q_t(i) (1 - q_t(i)) \frac{\ell_t(i)^2 \mathbf{1}[ i \in \Nout(I_t) ] }{P_t(i)^2}
  \right]
  \\
  &
  \nonumber
  =
  \E \left[
    \sum_{i \in S} q_t(i) (1 - q_t(i)) \frac{\ell_t(i)^2 }{P_t(i)}
  \right]
  \leq
  \E \left[
    \sum_{i \in S} q_t(i) \frac{1 - q_t(i)}{P_t(i)}
  \right]
  \\
  &
  =
  \E \left[
    \sum_{i \in S} q_t(i) \frac{1 - q_t(i)}{1 - p_t(i)}
  \right]
  \leq
  \E \left[
    2 \sum_{i \in S} q_t(i) 
  \right]
  \leq 2,
  \label{eq:boundStabilityStrong1}
\end{align}
where the first equality follows from \eqref{eq:defhatell},
the third equality follows from the assumption of strong observability implying that $\Nin(i) = [K] \setminus \{ i \}$ for all $i \in S$,
and the second inequality follows from the second part of \eqref{eq:defFTRL} and the assumption of $\gamma_t \in [0, 0.5]$.
The second term of the right-hand side of \eqref{eq:boundStabilityStrong0}
is bounded as
\begin{align}
  \sum_{i \in V \setminus S } q_t(i) \hat{\ell}_t(i)^2
  \leq
  \E \left[
    \sum_{i \in V \setminus S} q_t(i) \frac{1}{P_t(i)}
  \right]
  \leq
  2
  \E \left[
    \sum_{i \in V \setminus S} p_t(i) \frac{1}{P_t(i)}
  \right]
  \leq
  8 \alpha(G) \ln \frac{K^2}{4 \gamma_t},
  \label{eq:boundStabilityStrong2}
\end{align}
where the second inequality follows from
the second part of \eqref{eq:defFTRL} and 
the assumption of $\gamma_t \in [0, 0.5]$,
and the last inequality follows from Lemma 5 by \citet{alon2015online}.

% From Lemma~\ref{lem:boundBregStrong} and arguments in the proof of \citet[Lemma 4, Lemma 5, Theorem 2]{alon2015online},
Combining \eqref{eq:boundStabilityStrong0}, \eqref{eq:boundStabilityStrong1} and \eqref{eq:boundStabilityStrong2},
we obtain
\begin{align}
  \E
  \left[
  \linner \hat{\ell}_t, q_{t} - q_{t+1} \rinner 
  -
  D_t ( q_{t+1} , q_t )
  \right]
  \leq
  \frac{2}{\beta_t}
  \left( 
    1 + 4 \alpha(G) \ln \frac{K^2}{4 \gamma_t}
  \right) .
  \label{eq:boundStabilityStrong}
\end{align}
In addition,
from the definition of $\psi_t$ in \eqref{eq:defShannon},
we have
\begin{align*}
  &
  \sum_{t=1}^T \left(
    \psi_{t} (q_{t+1}) - \psi_{t+1}(q_{t+1})
  \right)
  +
  \psi_{T+1}(\mu_{i^*})
  -
  \psi_1(q_1)
  \\
  &
  =
  \sum_{t=1}^T 
  (\beta_{t+1} - \beta_{t}) H( q_{t+1} )
  -
  \beta_{T+1} H(\mu_{i^*})
  +
  \beta_1 H(q_1)
  \\
  &
  \leq
  \sum_{t=1}^T 
  (\beta_{t+1} - \beta_{t}) H( q_{t+1} )
  +
  \beta_1 \ln K .
\end{align*}
By combining this with \eqref{eq:boundStabilityStrong} and Lemma~\ref{lem:FTRL},
we obtain \eqref{eq:lemStrong}.

\subsection{Proof of Lemma~\ref{lem:boundat}}
\begin{proof}
  For any $p \in \cP(V)$,
  and for any $i^* \in V$,
  we have
  \begin{align}
    H(p)
    &
    =
    \sum_{i \in V} p (i) \ln \frac{1}{p(i)}
    =
    \sum_{i \in V \setminus \{ i^* \}} p (i) \ln \frac{1}{p(i)}
    +
    p(i^*) \ln \left( 1 + \frac{1-p (i^*)}{p(i^*)} \right)
    \nonumber
    \\
    &
    \leq
    (K-1)
    \cdot
    \frac{\sum_{i \in V \setminus \{ i^* \}} p (i)}{K-1} \ln 
    \frac{K-1}{\sum_{i \in V \setminus \{ i^* \}} p (i)}
    +
    p(i^*) \frac{1-p(i^*)}{p(i^*)}
    \nonumber
    \\
    &
    =
    (1 - p(i^*))
    \left(
      \ln \frac{K-1}{1 - p(i^*)}
      +
      1
    \right),
  \label{eq:boundH}
  \end{align}
  where the inequality follows from
  Jensen's inequality and $\ln ( 1 + x ) \leq x$ that holds for any $x \geq 0$
  and the last equality follows from $\sum_{i \in V} p(i) = 1$.
  Using this,
  we have
  \begin{align*}
    \sum_{t=1}^T a_t
    = 
    \sum_{t=1}^T H( q_t )
    &
    \leq
    \sum_{t=1}^T
    (1 - q_t(i^*))
    \left(
      \ln \frac{K-1}{1 - q_t(i^*)}
      +
      1
    \right)
    \\
    &
    \leq
    Q(i^*)
    \left(
      \ln \frac{(K-1)T}{Q(i^*)}
      +
      1
    \right)
    \leq
    Q(i^*)
    \left(
      \ln \frac{\mathrm{e}KT}{Q(i^*)}
    \right)
    ,
  \end{align*}
  where the second inequality follows from Jensen's inequality
with the definition $Q(i^*) = \sum_{t=1}^T ( 1 - q_t(i^*) )$.
\end{proof}

\subsection{Proof of Lemma~\ref{lem:decomposeweak}}
We use the following lemma to analyze the right-hand sided of \eqref{eq:lemFTRL}.
\begin{lemma}
  \label{lem:boundBregWeak}
  If $\psi_t$ is given by \eqref{eq:defpsiweak},
  it holds for any $\ell: V \rightarrow \re$ and
  $p, q \in \cP(V)$ that
  \begin{align}
    \nonumber
    \linner \ell, p - q \rinner 
    &
    -
    D_t ( q , p )
    \leq
    \beta_t
    \sum_{i \in V_1}
    \min \left\{
    p(i)
    \xi \left(
      \frac{\ell(i)}{\beta_t}
    \right),
    (1- p(i))
    \xi \left(
      -
      \frac{ \ell(i)}{\beta_t}
    \right)
    \right\}
    \\
    &
    +
    \sqrt{t}
    \sum_{i \in V_2}
    \min \left\{
    \sqrt{ p(i) }
    \zeta \left(
      \frac{\sqrt{p(i)} \ell(i)}{\sqrt{t}}
    \right),
    \sqrt{ 1- p(i) }
    \zeta \left(
      -
      \frac{ \sqrt{1 - p(i)} \ell(i)}{\sqrt{t}}
    \right)
    \right\},
    \label{eq:boundBregWeak}
    \\
    &
    \mbox{
    where
    }
    \quad
    \xi(x) = \exp(-x) + x - 1,
    \quad
    \zeta(x) = \frac{x^2}{1 + x}.
    \label{eq:defxizeta}
  \end{align}
\end{lemma}
\begin{proof}
  For any $x, y \in (0, 1)$,
  we define $d^{(1)}(y, x) \geq 0$ and $d^{(2)}(y, x) \geq 0$ by
  \begin{align}
    \label{eq:defd1}
    d^{(1)}(y, x)
    &
    =
    y \ln y - x \ln x - ( \ln x + 1 )(y - x) 
    =
    y \ln \frac{y}{x}
    +
    x - y
    ,
    \\
    d^{(2)}(y, x)
    &
    =
    - 2 \sqrt{y}
    + 2 \sqrt{x}
    + \frac{1}{\sqrt{x}} ( y - x)
    =
    \frac{1}{\sqrt{x}}
    \left(
      \sqrt{y} - \sqrt{x}
    \right)^2 .
    \label{eq:defd2}
  \end{align}
  Note that
  $d^{(1)}$ and $d^{(2)}$ correspond to
  Bregman divergences over $(0, 1)$ for 
  $\psi^{(1)}(x) = x \ln x$ and $\psi^{(2)}(x) = - 2 \sqrt{x}$.
  If $\psi_t$ is given by \eqref{eq:defpsiweak},
  the Bregman divergence $D_t( q, p)$ associated with $\psi_t$ is expressed as
  \begin{align*}
    D_t(q, p)
    &
    =
    \beta_{t}
    \sum_{i \in V_1} 
    \left( 
      d^{(1)}( q(i), p(i) )
      +
      d^{(1)}(1 -  q(i), 1 - p(i) )
    \right)
    \\
    &
    \quad
    +
    \sqrt{t}
    \sum_{i \in V_2} 
    \left( 
      d^{(2)}( q(i), p(i) )
      +
      d^{(2)}(1 -  q(i), 1 - p(i) )
    \right).
  \end{align*}
  From this,
  we have
  \begin{align}
    \nonumber
    &
    \linner \ell, p - q \rinner
    - D_t(q, p)
    \\
    &
    \nonumber
    \leq
    \sum_{i \in V_1} 
    \left(
      \ell(i) ( p(i) - q(i) )
      -
      \beta_{t} ( d^{(1)}( q(i), p(i) ) + d^{(1)}(1 - q(i), 1 - p(i)) )
    \right)
    \\
    &
    \nonumber
    \quad
    +
    \sum_{i \in V_2} 
    \left(
      \ell(i) ( p(i) - q(i) )
      -
      \sqrt{t} ( d^{(2)}( q(i), p(i) ) + d^{(2)}(1 - q(i), 1 - p(i)) )
    \right)
    \\
    &
    \nonumber
    \leq
    \sum_{i \in V_1} 
    \min
    \left\{
      \ell(i) ( p(i) - q(i) )
      -
      \beta_{t}
      d^{(1)}( q(i), p(i) )
      ,
      \ell(i) ( p(i) - q(i) )
      -
      \beta_{t}
      d^{(1)}(1 -  q(i), 1 - p(i) )
    \right\}
    \\
    &
    \quad
    +
    \sum_{i \in V_2} 
    \min
    \left\{
      \ell(i) ( p(i) - q(i) )
      -
      \sqrt{t}
      d^{(2)}( q(i), p(i) )
      ,
      \ell(i) ( p(i) - q(i) )
      -
      \sqrt{t}
      d^{(2)}(1 -  q(i), 1 - p(i) )
    \right\}.
    \label{eq:boundBregWeak0}
  \end{align}
  By the arguments in the proof of Lemma~\ref{lem:boundBregStrong},
  we have
  \begin{align}
    \ell(i) ( p(i) - q(i) )
    -
    \beta_{t}
    d^{(1)}( q(i), p(i) )
    \leq
    \beta_t
    p(i)
    \xi \left( 
      \frac{\ell(i)}{ \beta_t }
    \right) .
    \label{eq:boundBregWeakShanon0}
  \end{align}
  In a similar way,
  we can show
  \begin{align}
    \nonumber
    &
    \ell(i) ( p(i) - q(i) )
    -
    \beta_{t}
    d^{(1)}(1 -  q(i), 1 - p(i) )
    \\
    &
    =
    - \ell(i) ( (1 - p(i)) - (1 - q(i)) )
    -
    \beta_{t}
    d^{(1)}(1 -  q(i), 1 - p(i) )
    \leq
    \beta_t
    (1- p(i))
    \xi \left( 
      -
      \frac{\ell(i)}{ \beta_t }
    \right) .
    \label{eq:boundBregWeakShanon1}
  \end{align}
  Let us next evaluate 
  the term 
  $
    \ell(i) ( p(i) - q(i) )
    -
    \sqrt{t}
    d^{(2)}( q(i), p(i) )
  $ in the right-hand side of \eqref{eq:boundBregWeak0}.
  Denoting $z = \sqrt{q(i)}$,
  we have
  \begin{align}
    \ell(i) ( p(i) - q(i) )
    -
    \sqrt{t}
    d^{(2)}( q(i), p(i) )
    =
    \ell(i) (p (i) - z^2 )
    -
    \sqrt{t}
    \frac{1}{\sqrt{p(i)}}
    \left(
      z - \sqrt{p(i)}
    \right)^2 ,
    \label{eq:funcz}
  \end{align}
  where the last inequality follows from \eqref{eq:defd2}.
  Hence,
  its derivative in $z$ can be expressed as
  \begin{align}
    -
    2
    \ell(i) z
    -
    2 \sqrt{t}
    \frac{1}{\sqrt{p(i)}}
    \left(
      z - \sqrt{p(i)}
    \right)
    =
    -
    2 \left(\ell(i) + 
    \sqrt{
    \frac{t}{p(i)}
    }
    \right)
    z
    +
    2 \sqrt{t} .
  \end{align}
  The value of this expression is equal to zero
  when
  $z = z^* := \frac{\sqrt{t p(i)}}{ \sqrt{t}  + \sqrt{p(i)} \ell(i) }$.
  As \eqref{eq:funcz} is concave in $z$,
  its value is maximized when $z = z^*$.
  Hence,
  we have
  \begin{align}
    \nonumber
    &
    \ell(i) ( p(i) - q(i) )
    -
    \sqrt{t}
    d^{(2)}( q(i), p(i) )
    \leq
    \ell(i) (p (i) - z^{*2} )
    -
    \sqrt{t}
    \frac{1}{\sqrt{p(i)}}
    \left(
      z^* - \sqrt{p(i)}
    \right)^2 
    \\
    &
    \nonumber
    =
    \left(\sqrt{p(i)} - z^{*} \right)
    \left(
    \ell(i)
    \left( \sqrt{p(i)} + z^{*}\right)
    -
    \frac{ \sqrt{t} }{\sqrt{p(i)}}
    \left( \sqrt{p(i)} - z^{*} \right)
    % ( z^{*} - \sqrt{p(i)})
    \right)
    \\
    &
    \nonumber
    =
    \frac{{p(i)} \ell(i)}{ \sqrt{t} + \sqrt{p(i)} \ell(i) }
    \left(
    \ell(i)
    % \left( 
    \sqrt{p(i)}
    +
    \left(
      \ell(i)
      +
      \frac{\sqrt{t}}{\sqrt{p(i)}}
    \right)
    z^{*}
    -
    \sqrt{t}
    % ( z^{*} - \sqrt{p(i)})
    \right)
    \\
    &
    =
    \frac{{p(i)} \ell(i)}{ \sqrt{t} + \sqrt{p(i)} \ell(i) }
    \ell(i)
    \sqrt{p(i)}
    =
    \sqrt{p(i)} \frac{\left(\sqrt{p(i)}\ell(i) \right)^2}{\sqrt{t} + \sqrt{p(i)}\ell(i)}
    =
    \sqrt{t p(i)}
    \zeta \left(
    \frac{\sqrt{p(i)} \ell(i)}{\sqrt{t}}
    \right) .
    \label{eq:boundBregWeakTsallis0}
  \end{align}
  In a similar way to that for showing \eqref{eq:boundBregWeakTsallis0},
  we can show
  \begin{align}
    \nonumber
    &
    \ell(i) ( p(i) - q(i) )
    -
    \sqrt{t}
    d^{(2)}(1 -  q(i), 1 - p(i) )
    \\
    &
    \nonumber
    =
    -
    \ell(i) ( (1 - p(i)) - (1 - q(i)) )
    -
    \sqrt{t}
    d^{(2)}(1 -  q(i), 1 - p(i) )
    \\
    &
    \leq
    \sqrt{t (1 - p(i))}
    \zeta \left(
      -
    \frac{\sqrt{1 - p(i)} \ell(i)}{\sqrt{t}}
    \right) .
    \label{eq:boundBregWeakTsallis1}
  \end{align}
  Combining \eqref{eq:boundBregWeak0},
  \eqref{eq:boundBregWeakShanon0},
  \eqref{eq:boundBregWeakShanon1},
  \eqref{eq:boundBregWeakTsallis0} and
  \eqref{eq:boundBregWeakTsallis1},
  we obtain \eqref{eq:boundBregWeak}.
\end{proof}
Note that $\xi(x)$ and $\zeta(x)$ defined in \eqref{eq:defxizeta} satisfy
$\xi(x) \leq x^2$ for $x \geq -1$ and
$\zeta(x) \leq 2 x^2$ for $x \geq - \frac{1}{2}$.

Using Lemma~\ref{lem:boundBregWeak},
we evaluate
$\linner \hat{\ell}_t, q_{t} - q_{t+1} \rinner - D_t(q_{t+1}, q_t)$.
As we define $p_t$ by \eqref{eq:defFTRL} with $U = D$,
we have $p_t(i) \geq \frac{\gamma_t}{|D|}$ for all $i \in D$.
Hence,
for any $i \in V_1 = \bigcup_{ j \in D } \Nout ( j )$,
the value of $P_t(i) $ defined by in \eqref{eq:defhatell} is bounded as
\begin{align}
  P_t (i)
  =
  \sum_{j \in \Nin (i)} p_t(j)
  \geq
  \frac{\gamma_t}{|D|},
  \label{eq:boundPtWeak}
\end{align}
% at least $\frac{\gamma_{t}}{|D|}$,
which implies $\hat{\ell}_t \leq \frac{\ell_t(i)}{ P_{t}(i) } \leq \frac{|D|}{\gamma_t}$.
From this and the assumption of
$\gamma_t \geq \frac{2 |D|}{\beta_t}$,
we have
$
\frac{\hat{\ell}_t(i)}{\beta_t}
\leq
\frac{|D|}{\beta_t \gamma_t }
\leq
\frac{1}{2}
$
for all $i \in V_1$.
As we have
$\zeta(x) \leq x^2 $ for $x \leq - \frac{1}{2}$,
it holds for any $i \in V_1$ that
\begin{align}
  &
  \E \left[
    \min \left\{
      q_t(i)
      \xi \left(
        \frac{\hat{\ell}_t(i)}{\beta_t}
      \right)
      ,
      (1 - q_t(i))
      \xi \left(
        -
        \frac{\hat{\ell}_t(i)}{\beta_t}
      \right)
    \right\}
  \right]
  \nonumber
  \\
  &
  \leq
  \E \left[
    \min \left\{
      q_t(i),
      (1 - q_t(i))
    \right\}
    \left(
      \frac{\hat{\ell}_t(i)}{\beta_t}
    \right)^2
  \right]
  \nonumber
  \\
  &
  =
  \E \left[
    \min \left\{
      q_t(i),
      (1 - q_t(i))
    \right\}
    \left(
      \frac{{\ell}_t(i)^2 \mathbf{1}\left[ i \in \Nout(I_t) \right]}{ P_t(i)^2 \beta_t}
    \right)^2
  \right]
  \nonumber
  \\
  &
  =
  \E \left[
    \min \left\{
      q_t(i),
      (1 - q_t(i))
    \right\}
      \frac{{\ell}_t(i)^2 }{ P_t(i) \beta_t^2}
  \right]
  \leq
  \E \left[
    \frac{2|D|}{\beta_t^2 \gamma_t}
    % \min \left\{
      q_t(i)
      (1 - q_t(i))
    % \right\}
  \right],
  \nonumber
\end{align}
where the last inequality follows from \eqref{eq:boundPtWeak}
and the inequality  $\min \{x, 1 - x \} \leq 2 x (1-x) $ that holds for any $x \in [0, 1]$.
We hence have
\begin{align}
  \E \left[
    \sum_{i\in V_1}
    \min \left\{
      q_t(i)
      \xi \left(
        \frac{\hat{\ell}_t(i)}{\beta_t}
      \right)
      ,
      (1 - q_t(i))
      \xi \left(
        -
        \frac{\hat{\ell}_t(i)}{\beta_t}
      \right)
    \right\}
  \right]
  &
  \leq
  \E \left[
    \frac{2|D|}{\beta_t \gamma_t}
    % \min \left\{
      \sum_{i \in V_1}
      q_t(i)
      (1 - q_t(i))
    % \right\}
  \right]
  \nonumber
  \\
  &
  =
  \E \left[
    \frac{2|D|b_t}{\beta_t^2 \gamma_t}
  \right] .
  \label{eq:boundStabilityWeak}
\end{align}

For any $i \in V_2$,
we have $i \in \Nin (i)$,
which implies
$P_t(i) \geq p_t(i) \geq (1-\gamma_t) q_t(i) \geq \frac{1}{2} q_t(i)$.
We hence have
\begin{align}
  \E \left[
    \zeta \left(
      \frac{\sqrt{q_t(i)} \hat{\ell}_t(i)}{\sqrt{t}}
    \right)
    % |
    % p_t
  \right]
  &
  \leq
  \E \left[
    \zeta \left(
      \frac{\sqrt{q_t(i)} \hat{\ell}_t(i)}{\sqrt{t}}
    \right)
    % |
    % p_t
  \right]
  \leq
  \E \left[
    \left(
      \frac{\sqrt{q_t(i)} \hat{\ell}_t(i)}{\sqrt{t}}
    \right)^2
    % |
    % p_t
  \right]
  \nonumber
  \\
  &
  =
  \E \left[
    \frac{ q_t(i)}{t}
    \frac{ \ell_t(i)^2 \mathbf{1}[ i \in \Nout(I_t) ]}{P_t(i)^2}
    % |
    % p_t
  \right]
  \leq
  \E \left[
    \frac{ q_t(i)}{t P_t(i)}
    % |
    % p_t
  \right]
  \leq
  \frac{2}{t }.
  \label{eq:boundStabilityTsallis1}
\end{align}
Further,
if $q_t(i) \geq \frac{15}{16}$,
we have
$
\frac{\sqrt{1 - q_t(i)}\hat{\ell}_t(i)}{\sqrt{t}} 
\leq
\frac{1}{4 P_t(i) \sqrt{t}} 
\leq
\frac{1}{2 q_t(i) \sqrt{t}} 
\leq
\frac{8}{15} 
$.
As $\zeta(x)$ satisfies
$\zeta(x) \leq \frac{x^2}{1 + x} \leq \frac{15}{7} x^2$ for 
any $x \geq - \frac{8}{15}$,
we have
% This implies that
\begin{align}
  \nonumber
  &
  \zeta \left(
    -
    \frac{\sqrt{1 - q_t(i)}\hat{\ell}_t(i)}{\sqrt{t}} 
  \right)
  \leq
  \frac{15}{7}
  \left(
    \frac{\sqrt{1 - q_t(i)}\hat{\ell}_t(i)}{\sqrt{t}} 
  \right)^2
  =
  \frac{15}{7}
  \frac{1 - q_t(i)}{t}
  \frac{\ell_t(i)^2 \mathbf{1}[i \in \Nout(I_t)]}{P_t(i)^2} 
  \\
  &
  \leq
  \frac{60}{7}
  \frac{1 - q_t(i)}{t}
  \frac{\mathbf{1}[i \in \Nout(I_t)]}{q_t(i)^2} 
  \leq
  \frac{60}{7}
  \left(
    \frac{16}{15}
  \right)^2
  \frac{1 - q_t(i)}{t}
  \leq
  10
  \frac{1 - q_t(i)}{t}
  \label{eq:boundStabilityTsallis2}
\end{align}
if $i \in V_2$ and $q_t(i) \geq \frac{15}{16}$.
From \eqref{eq:boundStabilityTsallis1} and \eqref{eq:boundStabilityTsallis2},
for $i \in V_2$,
we have
\begin{align}
  \nonumber
  &
  \E
  \left[
    \min \left\{
    \sqrt{ q_t(i) }
    \zeta \left(
      \frac{\sqrt{q_t(i)} \hat{\ell}_t(i)}{\sqrt{t}}
    \right),
    \sqrt{ 1 - q_t(i) }
    \zeta \left(
      -
      \frac{ \sqrt{1 - q_t(i)} \hat{\ell}_t(i)}{\sqrt{t}}
    \right)
    \right\}
    |
    q_t(i)
  \right]
  \\
  &
  \leq
  \left\{
    \begin{array}{ll}
      2 \frac{\sqrt{q_t(i)}}{t} & \left( q_t(i) < \frac{15}{16} \right)
      \\
      10 \frac{1 - q_t(i)}{t} & \left( q_t(i) \geq \frac{15}{16} \right)
    \end{array}
  \right.
  % \\
  % &
  =
  O \left(
  \frac{1}{t} \sqrt{q_t(i) ( 1 - q_t(i) )}
  \right) .
  \label{eq:boundStabilityTsallis}
\end{align}
% From this,
% we have
% \begin{align}
%   \E
%   \left[
%     \sum_{i \in V_2}
%     \min \left\{
%     \sqrt{ q_t(i) }
%     \zeta \left(
%       \frac{\sqrt{q_t(i)} \hat{\ell}_t(i)}{\sqrt{t}}
%     \right),
%     \sqrt{ 1 - q_t(i) }
%     \zeta \left(
%       -
%       \frac{ \sqrt{1 - q_t(i)} \hat{\ell}_t(i)}{\sqrt{t}}
%     \right)
%     \right\}
%     |
%     p_t
%   \right]
% \end{align}

We further have
\begin{align}
  \nonumber
  &
  \sum_{t=1}^T
  \left( \psi_t(q_{t+1}) - \psi_{t+1}(q_{t+1}) \right)
  +
  \psi_{T+1} (\mu_{i^*})
  -
  \psi_1(q_1)
  \\
  &
  \nonumber
  =
  \sum_{i \in V_1} \left(
  \sum_{t=1}^T
  \left( \beta_t - \beta_{t+1} \right) h(q_{t+1}(i))
  \right)
  +
  \sum_{i \in V_2} \left(
  \sum_{t=1}^T
  \left( \sqrt{t} - \sqrt{t+1} \right) g(q_{t+1}(i))
  \right)
  \\
  &
  \nonumber
  \quad
  - 2
  \sqrt{T+1} \cdot |V_2| 
  +
  \beta_1 
  \sum_{i \in V_1}
  h(q_1(i))
  +
  2
  \sum_{i \in V_2}
  g(q_1(i))
  \\
  &
  \nonumber
  =
  \sum_{t=1}^T
  \left( \beta_t - \beta_{t+1} \right) a_{t+1}
  +
  2
  \sum_{i \in V_2} \left(
  \sum_{t=0}^T
  \left( \sqrt{t+1} - \sqrt{t} \right) \left(\sqrt{q_{t+1}(i)} + \sqrt{1 - q_{t+1}(i)} - 1 \right)
  \right)
  +
  \beta_1 a_1
  \\
  &
  \leq
  \sum_{t=1}^T
  \left( \beta_t - \beta_{t+1} \right) a_{t+1}
  +
  \beta_1 a_1
  +
  2
  \sum_{t=1}^{T+1}
  \frac{1}{\sqrt{t}}
  \sum_{i \in V_2} 
  \sqrt{ 
    q_{t}(i)(1 - q_{t}(i))
  },
  \label{eq:boundPenaltyWeak}
\end{align}
where $a_t$ and $b_t$ are defined by \eqref{eq:defatbtweak} and
the last inequality follows from 
$\sqrt{t+1} - \sqrt{t} \leq \frac{1}{\sqrt{t+1}}$ and 
$\sqrt{x} + \sqrt{1 - x} - 1 \leq \sqrt{x(1-x)}$.
From Lemma~\ref{lem:FTRL} combined with \eqref{eq:boundStabilityWeak},
\eqref{eq:boundStabilityTsallis} and \eqref{eq:boundPenaltyWeak},
we have
\begin{align}
  R_T
  =
  O \left(
    \sum_{t=1}^T
    \left(
      \gamma_t
      +
      \frac{|D|b_t}{\beta_t \gamma_t}
      +
      (\beta_t - \beta_{t+1 })a_{t+1}
      +
      \frac{1}{\sqrt{t}}
      \sum_{i \in V_2}
      \sqrt{ 
        q_{t}(i)(1 - q_{t}(i))
      }
    \right)
    +
    \beta_1 a_1
  \right) .
\end{align}

\subsection{Proof of Proposition~\ref{prop:weak}}
\begin{proof}
  We note that
  $b_t \leq 1$ and $b_t \leq a_t \leq 2 \ln K $.
  We define $z_t = \frac{b_t a_{t+1}}{\gamma'_t}$
  and $Z_t = \sum_{s=1}^t z_s$.
  Then,
  from the definition of $\gamma'_t$,
  we have
  \begin{align}
    z_t
    =
    \frac{a_{t+1}b_t}{\gamma'_t}
    =
    4
    \frac{a_{t+1}}{c_1}
    \left(
    c_1
    +
    B_t^{1/3}
    \right)
    % \in
    % \left[
    \geq
    a_{t+1}
    \geq
    b_{t+1}
    %   % \frac{\ln K + 1}{c_1} 
    %   ( c_3 + t^{1/3} )
    % \right] ,
  \label{eq:zb}
  \end{align}
  where the second inequality follows from $b_t \leq a_t$.
  % and the last inequality follows from $c_1 \leq c_3$.
  % which follows from
  Further,
  % from
  % $a_{t+1} \leq 2 \ln K$ and
  % $c_1 \geq 2 \ln K$,
  we have
  % and $B_t \geq 0$.
  % Hence,
  % we have
  \begin{align}
    z_t
    =
    4
    \frac{a_{t+1}}{c_1} 
    \left(
      c_1 +
      % \left(
      %   \sum_{s=1}^t a_s
      % \right)
      B_t^{1/3}
    \right)
    \leq
    4
    \left(
      c_1 +
      % \left(
      %   \sum_{s=1}^t a_s
      % \right)
      B_t^{1/3}
    \right)
    \leq
    4
    c_1
    +
    4
    \left( 
      b_1
      +
      % \frac{c_1}{c_3}
      \sum_{s=1}^{t-1} z_s
    \right)^{1/3}
    \leq
    8
    \left(
    c_1
    +
    Z_{t-1}
    % \sum_{s=1}^{t-1} \zeta_s
    \right),
    % \frac{\ln K +1}{c_1} 
    % \left(
    %   c_3
    %   +
    %   \left(
    %   a_1
    %   +
    %   \frac{c_1}{c_3}
    %   \sum_{s=1}^{t-1} \zeta_s
    %   \right)^{1/3}
    % \right)
  \label{eq:zZ}
  \end{align}
  where the first inequality follows from
  $a_{t+1} \leq 2 \ln K$ and $c_1 \geq 2 \ln K$
  % the second inequality follows from \eqref{eq:zb} and $c_1 \leq c_4$,
  and the last inequality follows from
  $c_1 \geq 2$ and $b_1 \leq 1$.
  From this,
  we have
  \begin{align}
    \nonumber
    &
    \sum_{t=1}^T 
    ( \beta_{t+1} - \beta_t ) a_{t+1}
    =
    c_2
    \sum_{t=1}^T 
    \frac{z_{t}}{ \sqrt{ c_1  +  Z_{t-1} } }
    =
    4 c_2
    \sum_{t=1}^T 
    \frac{Z_{t} - Z_{t-1}}{ 3 \sqrt{ c_1  +  Z_{t-1} } + \sqrt{c_1 + Z_{t-1}} }
    \\
    &
    \leq
    4 c_2
    \sum_{t=1}^T 
    \frac{Z_{t} - Z_{t-1}}{ \sqrt{ c_1  +  Z_{t} } + \sqrt{c_1 + Z_{t-1}} }
    =
    4 c_2
    \sum_{t=1}^T 
    \left( \sqrt{ c_1  +  Z_{t} } - \sqrt{c_1 + Z_{t-1}} \right)
    % =
    \leq
    % O \left(
      4
      c_2
      \sqrt{ 
        % \sum_{t=1}^T \zeta_t 
        Z_T
      },
      \label{eq:BB1}
    % \right) .
    % \\
    % =
    % O \left(
    %   \frac{c_2}{\sqrt{c_1}}
    %   \sqrt{ \sum_{t=1}^T a_t  \left( 1 + \sum_{t=1}^T b_t \right)^{1/3}  }
    % \right).
  \end{align}
  where the first equality follows from the definitions of $\beta_t$ and $z_t$,
  and the first inequality follows from \eqref{eq:zZ}.

  We define $w_t = \frac{b_t}{\gamma'_t}$ and $W_t = \sum_{s=1}^t w_s$.
  From the definition of $\gamma'_t$,
  we have
  \begin{align}
    w_t
    =
    \frac{b_t}{\gamma'_t}
    =
    % \frac{c_3}{c_1}
    4
    \left(
    1
    +
    \frac{1}{c_1}
    B_t^{1/3}
    \right)
    \geq
    % 1,
    % \frac{c_3}{c_1}.
    4 .
    \label{eq:boundw}
  \end{align}
  Further,
  we have
  \begin{align}
    w_{1} \leq 8 ,
    \quad
    w_{t+1} = 
    4\left(
    1 +
    \frac{1}{c_1}
    B_{t+1}^{1/3}
    \right)
    \leq 
    4\left(
    1 +
    \frac{1}{c_1}
    (B_{t} + 1)^{1/3} 
    \right)
    \leq
    2 w_t ,
    \quad
    w_{t} \leq
    4 \left(
      1 + t^{1/3}
    \right) .
    \label{eq:boundw2}
  \end{align}
  % $w$
  % where the last inequality follows from $c_1 \leq c_3$.
  Then $\beta_t$ can be bounded as
  \begin{align*}
    \beta_t
    &
    =
    % \frac{c_2 c_3}{c_1}
    c_2
    +
    c_2
    \sum_{s=1}^{t-1} \frac{w_s}{ \sqrt{ c_1 + Z_{s-1} } }
    \geq
    \frac{c_2}{\sqrt{
      c_1 + Z_t
    }}
    % \left( c_4 +  \sum_{s=1}^t \zeta_s \right)^{-1/2}
    \left(
      % \frac{c_3}{c_1}
      1
      +
      \sum_{s=1}^{t-1} w_s
      % \frac{b_s}{\gamma_s'}
    \right)
    \\
    &
    =
    \frac{c_2}{\sqrt{
      c_1 + Z_t
    }}
    \left(
      % \frac{c_3}{c_1}
      1
      +
      W_{t-1}
    \right)
    \geq
    % \frac{c_2}{\sqrt{
    %   c_4 + Z_t
    % }}
    % \left(
    %   % \frac{c_3}{c_1}
    %   1
    %   +
    %   \frac{c_3}{c_1} (t - 1)
    %   % W_{t-1}
    % \right)
    % =
    \frac{c_2 t}{\sqrt{
      c_1 + Z_t
    }} ,
    % \left(
    %   1
    %   +
    %   \frac{c_3}{c_1} (t-1)
    % \right) .
  \end{align*}
  where the second inequality follows from \eqref{eq:boundw}.
  Hence,
  we have
  \begin{align}
    \sum_{t=1}^T
    \frac{b_t}{\gamma_t \beta_t}
    \leq
    \sum_{t=1}^T
    \frac{b_t}{\gamma'_t \beta_t}
    \leq
    \sum_{t=1}^T
    \frac{\sqrt{c_1 + Z_t}}{c_2}
    \frac{w_t}{1 + W_{t-1}}
    \leq
    \frac{\sqrt{c_1 + Z_T}}{c_2}
    \sum_{t=1}^T
    \frac{w_t}{1 + W_{t-1}}
    \\
    \leq
    O \left(
      \frac{\sqrt{c_1 + Z_T}}{c_2}
      \ln \left(
        1 + W_{T}
      \right)
    \right)
    \leq
    O \left(
      \frac{\sqrt{c_1 + Z_T}}{c_2}
      \ln T
    \right),
    \label{eq:BB2}
  \end{align}
  where the last inequality follows from \eqref{eq:boundw2} and
  the fourth inequality can be shown by taking the sum of the following inequality:
  \begin{align*}
    \ln (1 + W_{t}) - \ln (1  + W_{t-1} )
    =
    \ln 
      \frac{1 + W_t}{ 1 + W_{t-1}}
    =
    \ln \left( 
      1 + 
      \frac{w_t}{ 1 + W_{t-1}}
    \right)
    \geq
    \frac{1}{4}
    \cdot
    \frac{w_t}{ 1 + W_{t-1}},
  \end{align*}
  where the inequality follows from the facts that 
  $\ln(1+x) \geq \frac{1}{4} x$ holds for any $x \in [0, 8]$
  and that \eqref{eq:boundw2} implies
  $ \frac{w_t}{ 1 + W_{t-1}} \leq 8$ for all $t$.
  We further have
  \begin{align}
    \sum_{t=1}^T
    \frac{1}{\beta_t}
    \leq
    \sum_{t=1}^T
    \frac{ \sqrt{c_1 + Z_t}}{c_2  t}
    \leq
    \frac{ \sqrt{c_1 + Z_T}}{c_2 }
    \sum_{t=1}^T
    \frac{1}{t}
    =
    O \left(
      \frac{\sqrt{c_1 + Z_T}}{c_2 }
      \ln T
    \right) .
    \label{eq:BB3}
    % \sum_{t=1}^T
    % \frac{1}{t}
    % \frac{1}{1 + \frac{c_3}{c_1}(t-1)}
  \end{align}
  In addition,
  we have
  \begin{align}
    \sum_{t=1}^T \gamma'_{t}
    \leq
    \sum_{t=1}^T
    \frac{b_t}{c_1 + B_t^{1/3} }
    \leq
    \frac{3 c_1}{2}
    \sum_{t=1}^T
    \left(
      B_{t}^{2/3} - B_{t-1}^{2/3}
    \right)
    \leq
    \frac{3 c_1}{2}
    B_{T}^{2/3} 
    \label{eq:BB4}
    % \frac{}
  \end{align}
  where the first inequality follows from
  $y^{2/3} - x^{2/3} \geq \frac{2}{3} ( y - x ) y^{-1/3} $,
  which holds for any $y \geq x > 0$.
  Combining \eqref{eq:BB1}, \eqref{eq:BB2}, \eqref{eq:BB3} and \eqref{eq:BB4},
  we obtain
  \begin{align*}
    &
    \sum_{t=1}^T
    \left(
      \gamma_t + 
      \frac{\delta b_t}{\gamma_t \beta_t}
      +
      (\beta_{t+1} - \beta_t) a_{t+1}
    \right)
    =
    \sum_{t=1}^T
    \left(
      \gamma'_t 
      +
      \frac{2 \delta}{\beta_t}
      +
      \frac{\delta b_t}{\gamma_t \beta_t}
      +
      (\beta_{t+1} - \beta_t) a_{t+1}
    \right)
    \\
    &
    =
    O\left(
      c_1 B_T^{2/3}
      +
      \left(
        % \left(
        %   \frac{1}{c_2}
        %   +
        %   \frac{c_1}{c_2 c_3}
        % \right)\delta \ln T
        \frac{ \delta \ln T }{c_2}
        +
        c_2
      \right)
      \sqrt{
        c_1
        +
        Z_T
        % \sum_{t=1}^T \zeta_t
      }
    \right)
    \\
    &
    =
    O\left(
      c_1 B_t^{2/3}
      +
      \left(
        % \left(
          % \frac{1}{c_2}
          % +
          % \frac{c_1}{c_3}
        % \right)
        \frac{ \delta \ln T }{c_2}
        +
        c_2
      \right)
      \sqrt{
        c_1
        +
        \sum_{t=1}^T \frac{a_{t+1} }{c_1}
        \left(
          c_1
          +
          % \left(\sum_{t=1}^T b_t \right)
          B_t^{1/3}
        \right)
      }
    \right)
    \\
    &
    =
    O\left(
      c_1 B_t^{2/3}
      +
      \frac{1}{\sqrt{c_1}}
      \left(
        % \left(
          % \frac{1}{c_2}
          % +
          % \frac{c_1}{c_3}
        % \right)
        \frac{ \delta \ln T }{c_2}
        +
        c_2
      \right)
      \sqrt{
        c_1 ^2
        +
        % \frac{1}{c_1}
        \left(
          \ln K
          +
          A_T
        \right)
        \left(
          c_1
          +
          % \left(\sum_{t=1}^T b_t \right)
          B_T^{1/3}
        \right)
      }
    \right),
  \end{align*}
  where the third equality follows from \eqref{eq:zb}
  and the last equality follows from $a_{T+1} = O(\ln K)$.
\end{proof}

\subsection{Proof of Lemma~\ref{lem:boundATBT}}
\begin{proof}
  From the definition of $h(x)$,
  it holds for any $p \in \cP(V)$ and $i^* \in [K]$ that
  \begin{align}
    &
    -
    \sum_{i \in V_1}
    h ( p(i) )
    \leq
    -
    \sum_{i \in V}
    h ( p(i) )
    =
    \sum_{i \in V}
    \left(
      p(i)
      \ln
      \frac{1}{p(i)}
      +
      ( 1 - p(i) ) \ln \frac{1}{1 - p(i)}
    \right) 
    \nonumber
    \\
    &
    =
    H(p)
    +
    \sum_{i \in V}
    ( 1 - p(i) ) \ln \frac{1}{1 - p(i)}
    \leq
    (1 - p(i^*)) \ln \frac{\mathrm{e} K}{1 - p(i^*)} 
    +
    \sum_{i \in V}
    ( 1 - p(i) ) \ln \frac{1}{1 - p(i)} ,
    \label{eq:boundhp0}
    % \leq
    % \sum_{i \in V}
    % \left(
    %   p_i 
    %   \ln
    %   \frac{1}{p_i}
    %   +
    %   ( 1 - p_i ) \frac{1}{1 - p_i}
    % \right).
  \end{align}
  where the last inequality follows from \eqref{eq:boundH}.
  We further have
  \begin{align}
    \nonumber
    \sum_{i \in V} ( 1 - p(i) )  \ln \frac{1}{1 - p(i)}
    =
    (  1 - p(i^*)) \ln \frac{1}{1 - p(i^*)}
    +
    \sum_{i \in V \setminus \{ i^* \} } (1 - p(i))
    \ln \left(
      1 + \frac{p(i)}{1 - p(i)}
    \right)
    \\
    \leq
    (  1 - p(i^*)) \ln \frac{1}{1 - p(i^*)}
    +
    \sum_{i \in V \setminus \{ i^* \} } (1 - p(i))
    \left(
      \frac{p(i)}{1 - p(i)}
    \right)
    =
    (  1 - p(i^*)) \left( \ln \frac{1}{1 - p(i^*)} + 1 \right) .
    \label{eq:boundhp1}
  \end{align}
  Combining \eqref{eq:boundhp0} and \eqref{eq:boundhp1},
  we obtain
  \begin{align*}
    -
    \sum_{i \in V_1}
    h ( p(i) )
    \leq
    2 
    (  1 - p(i^*)) \ln \frac{\mathrm{e} K}{1 - p(i^*)} .
  \end{align*}
  From this,
  we have
  \begin{align*}
    A_T
    =
    -
    \sum_{t=1}^T
    \sum_{i \in V_1}
    h ( q_t(i) )
    \leq
    2 
    \sum_{t=1}^T
    (  1 - q_t(i^*)) \ln \frac{\mathrm{e} K}{1 - q_t(i^*)} 
    \leq
    2 
    Q( i^* ) \ln \frac{ \mathrm{e} K T }{ Q(i^*) },
  \end{align*}
  where the last inequality follows from the similar argument to Lemma~\ref{lem:boundat}.
  We also have
  \begin{align*}
    B_T
    &
    \leq
    \sum_{t=1}^T
    \sum_{i \in V}
    q_t(i) (1 - q_t(i) )
    % \\
    % &
    =
    \sum_{t=1}^T
    \left(
      q(i^*)(1 - q_t(i^*))
      +
      \sum_{i \in V \setminus \{i^*\}}
      q_t(i) (1 - q_t(i^*))
    \right)
    \\
    &
    \leq
    \sum_{t=1}^T
    \left(
      (1 - q_t(i^*))
      +
      \sum_{i \in V \setminus \{i^*\}}
      q_t(i)
    \right)
    =
    2
    \sum_{t=1}^T
    \left(
      1 - q_t(i^*)
    \right)
    =
    2 Q(i^*) .
  \end{align*}
  for any $i^* \in [K]$.
  This completes that proof of Lemma~\ref{lem:boundATBT}.
\end{proof}

\subsection{Proof of Lemma~\ref{lem:boundsqrt}}
\begin{proof}
  We have
  \begin{align}
    % -
    % \sum_{t=1}^T \sum_{i \in V_2} 
    % \frac{1}{\sqrt{t}}
    % \sum_{i \in V_2} g( q_{t} (i) )
    % =
    &
    \sum_{t=1}^T 
    \frac{1}{\sqrt{t}}
    \sum_{i \in V_2} 
    \sqrt{
    q_t(i)( 1 - q_t(i) )
    }
    \leq
    \sum_{t=1}^T 
    \frac{1}{\sqrt{t}}
    \sqrt{
      |V_2|
      \sum_{i \in V_2} 
      q_t(i)( 1 - q_t(i) )
    }
    \nonumber
    \\
    &
    \leq
    \sqrt{
      \left(
      \sum_{t=1}^T 
      \frac{1}{t}
      \right)
      \left(
      |V_2|
      \sum_{t=1}^T 
      \sum_{i \in V_2} 
      q_t(i)( 1 - q_t(i) )
      \right)
    }
    \leq
    \sqrt{
      |V_2|
      ( \ln T + 1 )
      \sum_{t=1}^T 
      \sum_{i \in V_2} 
      q_t(i)( 1 - q_t(i) )
    },
    \label{eq:lemsqrt}
  \end{align}
  where inequalities follow from the Cauchy-Schwarz inequality.
  We further have
  \begin{align*}
    \sum_{t=1}^T 
    \sum_{i \in V_2} 
    q_t(i)( 1 - q_t(i) )
    \leq
    \sum_{t=1}^T 
    ( 1 - q_t(i^*) )
    +
    \sum_{t=1}^T 
    \sum_{i \in V_2 \setminus \{ i^* \}} 
    q_t(i)
    \leq
    2
    \sum_{t=1}^T 
    ( 1 - q_t(i^*) )
    = 2 Q(i^*) 
  \end{align*}
  for any $i^* \in [K]$.
  Combining this with \eqref{eq:lemsqrt},
  we obtain $R^{(2)}_T = O \left( \sqrt{ |V_2| \ln T \cdot Q } \right)$.
\end{proof}

% \revise{
\section{Comparison with the result by \citet{rouyer2022near}}
% }
\label{sec:rouyer}
While \citet{rouyer2022near} consider the same research question as this paper,
their approach is different from ours in the following points.
Their algorithm follows the approach by \citet{seldin2014one} and \citet{seldin2017improved},
in which the suboptimality gaps $\Delta_i$ are explicitly estimated.
In contrast,
our algorithms do not use explicit estimation for suboptimality gap,
and instead employ the self-bounding technique to lead to stochastic regret bounds,
similarly to the algorithms by \citet{zimmert2021tsallis,wei2018more}.
Due to these differences in algorithm design and regret analysis,
it seems difficult to integrate these algorithms or provide a unified analysis.
% These differences in algorithm design and regret analysis result in the following different strength:

The differences in results can be summarized as follows:
\begin{itemize}
  \item Advantage of our results:
  \begin{itemize}
    \item Covered classes of feedback graphs: We provide algorithms for both strongly observable graphs and weakly observable graphs. On the other hand, the algorithms by \citet{rouyer2022near} only deal with graphs with self-loops, which is a special case of strongly observable graphs.
    \item Our algorithms can also handle stochastic environments with adversarial corruptions.
    \item Our regret bounds for strongly observable graph depend on the independence number $\alpha$ while the algorithms by \citet{rouyer2022near} depend on strong independent number $\tilde{\alpha}$,
    which is the independence number of the subgraph consisting of bidirectional edges. In general $\alpha \leq \tilde{\alpha}$, and for symmetric graphs $\alpha = \alpha'$.
    We also note that, in some cases, there is a significant discrepancy between $\alpha$ and $\tilde{\alpha}$.
    For example, a directed graph $G=(V,E)$ defined by $V=[K]$, $E = \{ (i, j) \in V \times V \mid i \leq j \}$ has $\alpha=1$ and $\tilde{\alpha} = K$.
  \end{itemize}
  \item Advantage of results by \citet{rouyer2022near}:
  \begin{itemize}
    \item Their algorithm has a regret bound expressed with individual suboptimality gaps $\Delta_i$ for stochastic environments, while the regret bounds in this paper depend only on $\Delta_{\min} = \min_{i \in [K] \setminus \{ i^* \}} \Delta_{i}$.
    Consequently,
    if many actions $i$ have large suboptimality gaps $\Delta_i \gg \Delta_{\min}$,
    their algorithms will perform better.
    \item Their regret bound has an improved dependency on $\ln T$.
    More precisely,
    their stochastic regret bounds for problems with strongly observable graphs scale with $O( (\ln T)^2 )$,
    which is better than our regret bounds of $O((\ln T)^3)$.
    \item Their paper includes extension to time varying feedback graphs though our algorithms seem to be extendable in a similar way.
  \end{itemize}
\end{itemize}

% \revise{
  \section{An alternative algorithm for the weakly observable case}
% }
\label{sec:weak-strong}
In the weakly observable case,
as shown in Theorem~\ref{thm:weak-informal},
our regret bound for stochastic environments
include an $O(\frac{K' \ln T}{\Delta_{\min}})$-term,
where $K' = |V_2|$ is the number vertices that are not dominated by the weakly dominating set $D$.
When $T$ is sufficiently larger than other problem parameters,
this term is negligibly small compared to the other term $\frac{\delta(\ln T)^2}{\Delta^2_{\min}}$.
However,
if $K'$ is larger than $\frac{\delta \ln T}{\Delta_{\min}}$,
this $O(\frac{K' \ln T}{\Delta_{\min}})$-term can be dominant.
In such a case,
the regret upper bound may be improved by modifying the algorithm.
Roughly speaking,
by combining the approach to strongly observable case,
the $O(\frac{K' \ln T}{\Delta_{\min}})$-term can be replaced with
an $O( \frac{\alpha^{(2)} ( \ln T )^{3}}{\Delta_{\min}} )$-term,
where $\alpha^{(2)}$ is the independent number of the subgraph induced by $V_2$,
i.e.,
\begin{align}
  G_2 = (V_2, E \cap (V_2 \times V_2)),
  \quad
  \alpha^{(2)} = \alpha(G_2).
\end{align}
We here note that $G_2$ is a strongly observable graph with self-loops as $D$ is a weakly dominating set (Definition~\ref{def:alphadelta})
and $V_2 = V \setminus \bigcup_{i\in D} \Nout (i)$.
If $\alpha^{(2)} (\ln T)^2 \leq K'$,
the modified version provides a better regret bound.
The details of the modification are given below.

Consider the following regularizer function:
\begin{align}
  \nonumber
  \psi_t (p)
  =
  \beta_t^{(1)}
  \sum_{i \in V_1} h (p(i))
  +
  \beta_t^{(2)}
  \sum_{i \in V_2} h (p(i)),
  \quad
  \mbox{where}
  \quad
  h(x) = x \ln x + (1 - x) \ln (1 - x)  .
  \label{eq:defpsiweak-alt}
\end{align}
We define $\beta_t^{(1)}$ and $\gamma_t^{(1)}$ in the same way as \eqref{eq:defgammabetaweak} in Section~\ref{sec:weak} with repracement of
$c_1 := c_1^{(1)}$ and $c_2 := c_2^{(1)}$.
Similarly,
we define 
$\beta_t^{(2)}$ and $\gamma_t^{(2)}$ in a similar way as \eqref{eq:defbeta} in Section~\ref{sec:strong} with
$c_1 := c_1^{(2)}$ and $a_s := \sum_{i \in V_2} h (q_s (i)) $.
Parameters
$c_1^{(1)}$, $ c_2^{(1)}$ and $c_1^{(2)}$ are specified later.
Using this regularizer function,
we compute $q_t$ using FTRL given by \eqref{eq:defFTRL}.
Then,
we compute $p_t$ by
\begin{align}
  p_t = (1-\gamma_{t}^{(1)} - \gamma_{t}^{(2)}) q_t
  +
  \gamma_{t}^{(1)} \mu_{D}
  +
  \gamma_{t}^{(2)} \mu_{V_2} .
\end{align}
We then have the following regret bound:
\begin{align}
  \nonumber
  R_T 
  &
  \leq 
  \hat{c}^{(1)} \cdot \max \left\{ \bar{Q}^{2/3}, \left( c_1^{(1)} \right)^2 \right\} 
  +
  \hat{c}^{(2)} \cdot \max \left\{ \bar{Q}^{1/2}, 1 \right\} 
  \quad
  \mbox{where}
  \\
  &
  \quad
  \hat{c}^{(1)}
  =
  O \left(
    c_1^{(1)}
    +
    \frac{1}{\sqrt{c_1^{(1)}}}
    \left(
      \frac{ |D| \ln T}{c_2^{(1)}}
      +
      c_2^{(1)}
    \right)
    \sqrt{\ln (KT)}
  \right) ,
  \nonumber
  \\
  &
  \quad
  \hat{c}^{(2)}
  =
  O \left(
    \left(
    \frac{\alpha^{(2)} \ln T \cdot \ln (c_1^{(2)} K T ) }{c_1^{(2)} \sqrt{\ln K}} 
    +
    c_1^{(2)} 
    \sqrt{\ln K}
    \right)
    \sqrt{\ln(KT)}
  \right) .
  \label{eq:thmweak-alt}
\end{align}
Consequently,
in adversarial regimes with self-bounding constraints,
we have
\begin{align}
  \label{eq:thmweak-alt2}
  R_T = 
      O\left( 
        \frac{(\hat{c}^{(1)})^3}{ \Delta_{\min}^2 } + \left(\frac{C^2 (\hat{c}^{(1)})^3}{\Delta_{\min}^2} \right)^{1/3}
        +
        O\left(  \frac{(\hat{c}^{(2)})^2}{ \Delta_{\min} } + \sqrt{ \frac{C (\hat{c}^{2})^2}{\Delta_{\min}} } \right)
      \right)  .
\end{align}
Similarly to the analysis in Section~\ref{sec:weak},
we obtain
$
\hat{c}^{(1)}
=
O \left(
  \left(
    |D| \ln T \cdot \ln (KT)
  \right)^{1/3}
\right) 
$ 
by setting
$c_1^{(1)} = \Theta\left( \left( |D| \ln T \cdot \ln (KT)  \right)^{1/3} \right) $
and
$c_2^{(1)} = \Theta \left( \sqrt{ |D| \ln T } \right)$.
Further,
by setting
$c_1^{(2)} = \Theta \left(  \sqrt{ \frac{ \alpha^{(2)} \ln T \cdot \ln (KT) }{ \ln K } }  \right)$,
we obtain
$\hat{c}^{(2)}  =
O \left( \sqrt{ \alpha^{(2)} \ln T \cdot (\ln (KT))^2 } \right)$.

Consequently,
the modified algorithm achieves
\begin{align}
  R_T = |D|^{1/3} (T \ln T )^{2/3} + \sqrt{ \alpha^{(2)} T (\ln T)^3 }
\end{align}
for adversarial environments
and
\begin{align}
  R_T = 
  \frac{ |D| (\ln T )^2 }{\Delta_{\min}^2}
  +
  \left( \frac{C^2 |D| (\ln T )^2 }{\Delta_{\min}^2} \right)^{1/3}
  +
  \frac{\alpha^{(2)} (\ln T)^3 }{\Delta_{\min}}
  +
  \left(
  \frac{C \alpha^{(2)} (\ln T)^3}{\Delta_{\min}}
  \right)^{1/2}
\end{align}
for stochastic environments with adversarial corruptions (more generally, in adversarial regimes with self-bounding constraints).

% \revise{
\section{Note on the definition of weak domination}
% }
\label{sec:weak-domination-definition}
Previous studies,
e.g., \citet{alon2015online},
have adopted 
a slightly different definition of \textit{weak domination} rather than one in this paper:
\begin{definition}[alternative difitnition of weak domination, \citep{alon2015online}]
  \label{def:WDalt}
  For any directed graph $G=(V, E)$ with a set of weakly observable vertices $W \subseteq V$,
  a \textit{weakly observable set} $D' \subseteq V$ is a set of vertices that dominates $W$,
  i.e.,
  that satisfies $ W \subseteq \bigcup_{i \in D'} \Nout (i)$.
  The \textit{weak domination number} $\delta'(G)$ of G is the size of its smallest weakly dominating set.
\end{definition}
We can see that our definition of weakly dominating set in Definition~\ref{def:alphadelta} and that in Definition~\ref{def:WDalt} coincide,
with some very limited exceptions.
Consequently,
we will see that
$\delta(G)$ and $\delta'(G)$ in Definitions~\ref{def:alphadelta} and \ref{def:WDalt} satisfy
$\delta(G) \leq \delta'(G) \leq \delta(G) + 1$.
Further,
if $\delta(G) \geq 2$ then $\delta(G) = \delta'(G)$.
These facts can be confirmed as follows.

From the definition observability (Definition~\ref{def:observability}),
the vertices of a weakly observable graph are classified into the following three type:
\begin{description}
  \item[strongly observable vertices, type 1] $V_{\mathrm{SO1}} = \{ i \in V \mid i \in \Nin(i) \}$:
  vertices with self-loop
  (strongly observable vertices, type 1).
  \item[strongly observable vertices, type 2] $V_{\mathrm{SO2}} = \{ i \in V \mid \Nin(i) = V \setminus \{ i \} \}$:
  vertices without self-loop, with edges from all other vertices.
  (strongly observable vertices, type 2).
  \item[weakly observable vertices] $V_{\mathrm{WO}} = V \setminus ( V_{\mathrm{SO1}} \cup V_{\mathrm{SO2}} )$:
  weakly observable vertices.
\end{description}
Weakly dominating set $D$ in Definition~\ref{def:alphadelta} dominates
all vertices except $V_{\mathrm{SO1}}$,
i.e.,
all vertices in $V_{\mathrm{SO2}} \cup V_{\mathrm{WO}}$.
Weakly dominating set $D'$ in Definition~\ref{def:WDalt} dominates $V_{\mathrm{WO}}$.
It is clear that $D'$ dominates $V_{\mathrm{SO1}}$,
which means that $D'$ is a weakly dominating set in the sense of Definition~\ref{def:alphadelta} as well.
On the other hand,
if the size of $D$ is greater than or equal to $2$,
then it also dominates all vertices $V_{\mathrm{SO2}}$.
% as $\bigcup_{i \in D}$ dominates $V_{\mathrm{SO2}} \setminus \{ i_1 \}$  
This implies that $D$ is a weakly dominating set in the sense of Definition~\ref{def:WDalt} as well.
Therefore,
for vertex sets of size at least $2$,
the concept of weak domination is the same in Definition~\ref{def:alphadelta} as in Definition~\ref{def:WDalt}.
Consequently,
we have $\delta(G) = \delta'(G)$ if $\delta(G) \geq 2$.

The only exception is the case in which $|D| = 1$ and $D \subseteq V_{\mathrm{SO2}}$.
In this case,
however,
by adding an arbitrary vertex to $D$,
we can make it dominate $V_{\mathrm{SO2}}$ as well.
In other words,
for any $i \in V \setminus D$,
$D \cup \{ i \}$ dominates $V_{\mathrm{SO2}}$,
and hence,
is a weakly dominating set in the sense of Definition~\ref{def:WDalt} as well.
Hence,
even if $\delta(G) = 1$,
we have $1 \leq \delta'(G) \leq 2$.

\end{document}